\documentclass{article}

\usepackage{amsthm}
\usepackage{algorithm}
\usepackage{algorithmic}
\usepackage{amsmath}
\usepackage{mathtools}

\usepackage{microtype}
\usepackage{bm}
\usepackage{graphicx}
\usepackage{amsfonts}
\usepackage{tabularx}
\usepackage{multirow, makecell}
\usepackage{subcaption}
\usepackage{booktabs} % for professional tables

\usepackage{xcolor}
\usepackage[utf8]{inputenc}
\usepackage[english]{babel}
 
\newtheorem{proposition}{Proposition}

\usepackage{hyperref}

\DeclarePairedDelimiterX{\infdivx}[2]{(}{)}{%
  #1\;\delimsize\|\;#2%
}
\newcommand{\infdiv}{D_{KL}\infdivx}

\DeclareMathOperator{\E}{\mathbb{E}}
\DeclareMathOperator{\pmd}{p_{model}}

\DeclareMathOperator{\pgb}{p_{\text{global}}}

\renewcommand{\E}{\mathbb{E}}

\usepackage[pending]{icml2020}

\icmltitlerunning{Trading Off Diversity and Quality in Natural Language Generation}

\begin{document}

\twocolumn[
\icmltitle{Trading Off Diversity and Quality in Natural Language Generation}



\icmlsetsymbol{equal}{*}

\begin{icmlauthorlist}
\icmlauthor{Hugh Zhang}{equal,google,stanford}
\icmlauthor{Daniel Duckworth}{equal,google}
\icmlauthor{Daphne Ippolito}{google,upenn}
\icmlauthor{Arvind Neelakantan}{openai}
\end{icmlauthorlist}

\icmlaffiliation{google}{Google Brain}
\icmlaffiliation{openai}{OpenAI}
\icmlaffiliation{stanford}{Stanford University}
\icmlaffiliation{upenn}{University of Pennsylvania}

\icmlcorrespondingauthor{Hugh Zhang}{hughz@stanford.edu}

\icmlcorrespondingauthor{Daniel Duckworth}{duckworthd@google.com }
\icmlkeywords{Machine Learning, Natural Language Processing, ICML}

\vskip 0.3in
]



 \printAffiliationsAndNotice{\icmlEqualContribution} % otherwise use the standard text.
%\icmlEqualContribution
%%%%%%%%%%%%%%%%%%%%%%%%%%%%%%%%%%%%%%%%%%%%%%%%%%%%%%%%%%%%%%%%%%%%%%%%%%%%%%%%%%%%%%%%%%%%%%%%%%%
\begin{abstract}
For open-ended language generation tasks such as storytelling and dialogue, choosing the right decoding algorithm is critical to controlling the tradeoff between generation \emph{quality} and \emph{diversity}.
However, there presently exists no consensus on which decoding procedure is best or even the criteria by which to compare them. 
We address these issues by casting decoding as a multi-objective optimization problem aiming to simultaneously maximize both response quality and diversity. 
Our framework enables us to perform the first large-scale evaluation of decoding methods along the entire quality-diversity spectrum. 
We find that when diversity is a priority, all methods perform similarly, but when quality is viewed as more important, the recently proposed nucleus sampling \cite{holtzman2019curious} outperforms all other evaluated decoding algorithms.
Our experiments also confirm the existence of the `likelihood trap', the counter-intuitive observation that high likelihood sequences are often surprisingly low quality.
We leverage our findings to create and evaluate an algorithm called \emph{selective sampling} which tractably approximates globally-normalized temperature sampling.
%Additionally, we question the longstanding assumption that the quality of a sentence can be judged by its overall probability. Our experiments confirm the existence of a ‘likelihood trap’, the counter-intuitive observation that high likelihood sequences are often surprisingly low quality. Furthermore, we propose and evaluate selective sampling, a decoder which directly encodes this intuition by upweighting high probability sentences according to a globally normalized temperature sampling. When evaluating selective sampling alongside standard token-by-token decoders, we find that it performs poorly, even when taking the likelihood trap into account.
%Our experiments also confirm the existence of the `likelihood trap', the counter-intuitive observation that high likelihood sequences are often surprisingly low quality.
%We leverage our findings to create and evaluate an algorithm called \emph{selective sampling} which tractably approximates globally-normalized temperature sampling.
\end{abstract}

\section{Introduction}

Generative language models are applicable for a wide variety of tasks including writing articles, composing Shakespearean sonnets, or engaging in conversation.
For nearly all of these goals, human judgments are the sole way to credibly evaluate the quality of the generated text, rendering it prohibitively expensive to optimize directly over the desired objectives.
Researchers typically address this issue by adopting a two-stage process.
\begin{figure}[h]
    \centering
    \includegraphics[width=\columnwidth]{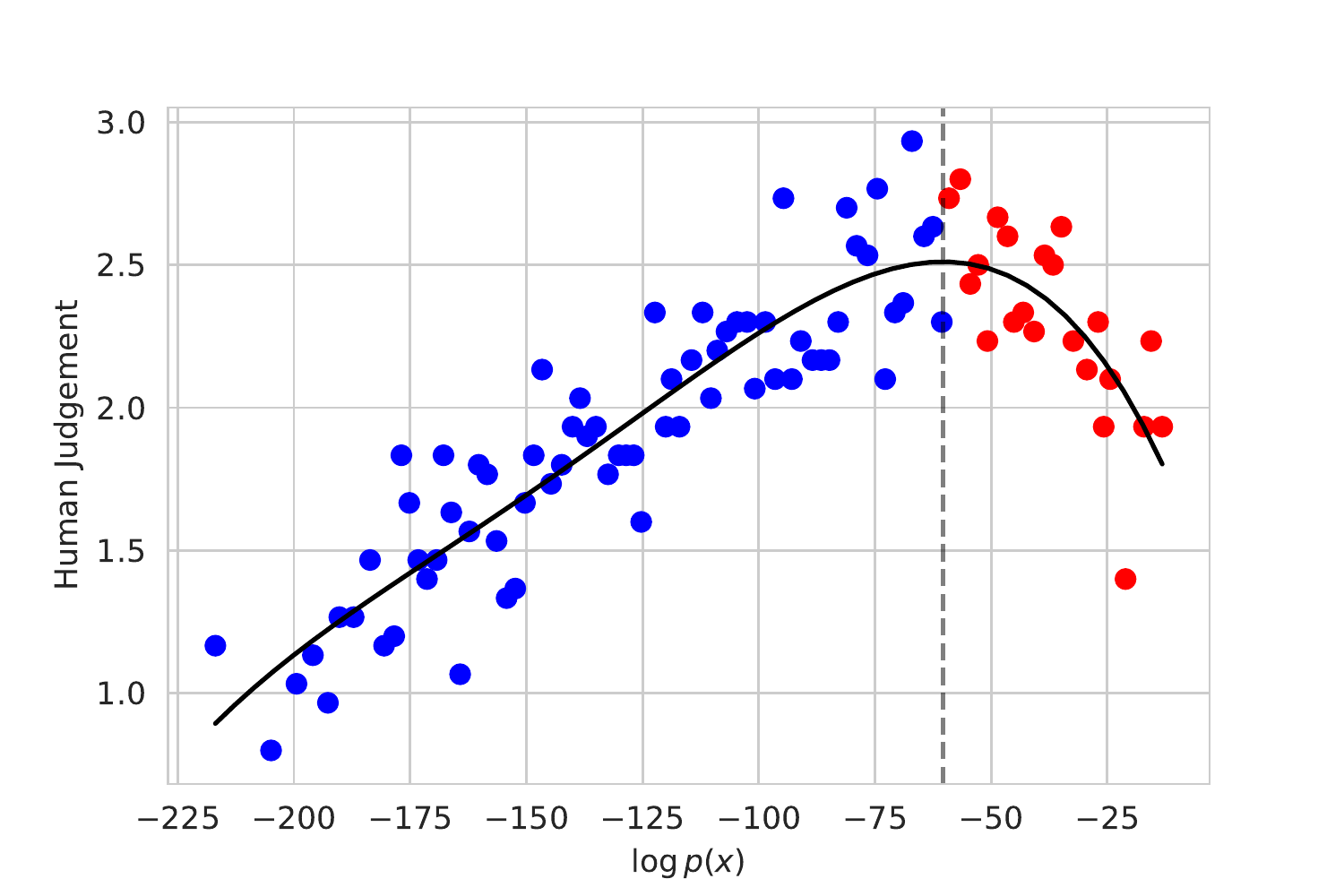}
    \caption{\textbf{The Likelihood Trap.} We asked 146 crowdworkers to rate the quality of 100 sentences across a variety of model likelihoods. While model log likelihoods are generally positively correlated with average human quality judgments, we notice an inflection point after which they become negatively correlated. Each point in the graph represents the average crowdworker rating of 5 sentences with similar model likelihoods. We discuss this phenomenon in more depth in Section~\ref{likelihoodtrap}.}
    \label{fig:likelihood-trap}
\end{figure}
At train time, models seek to imitate a human-written text corpus as a proxy for the true objective (e.g. higher quality samples).
At inference time, models generate text sequences via a decoding algorithm that better optimizes the desired success criteria given the original predictions from the network. 
Nearly all major breakthroughs in image and language generation over the past few years \cite{radford2019language, zhang2019dialogpt, fan2018hierarchical} have adopted this two stage process where the model probability distributions differ between train and inference time.

% where samples are generated differently and train and inference time.

% Start with two step, then importance of decoding, framework, lik. trap, global temp sampling.

% We want to optimize human judgment quality.
% This is infeasible.
% People have tried to improve quality of systems by changing architecture, training method, etc etc
% Effect of decoding algorithms is understudied despite their importance.
%
%Decoding algorithms, used to generate sequences of tokens from the probability distributions output by a language model, have been shown to
This work examines decoding methods for language models, which are well known to be critical for performance in language generation \citep{ippolito2019human}.
%; a poor choice of decoder can mean the difference between a model that outputs human-like natural text and one that outputs obviously machine-generated nonsense.
Recent efforts for improving generative language models models have focused primarily on altering the model architecture \citep{vaswani2017attention,gehring2017convolutional}, training method \citep{de2019training} and model size \citep{radford2019language,adiwardana2020towards}.
While effort has also been made towards improving decoders \cite{vijayakumar2016diverse, li2016mutual, ippolito2019comparison}, there has been significantly less progress towards evaluating improvements in decoder performance, especially for open-ended generative tasks where successful models must generate a \emph{diverse} spectrum of high \emph{quality} answers rather than merely a single output. 

For many tasks, these two criteria of quality and diversity are not always equally important. In machine translation, the most important criteria is to produce an accurate, high quality translation of the input; generating a variety of alternative translations is also useful, but not if it comes at the cost of correctness. Meanwhile, in open domain dialogue the goal is often to sustain an enjoyable conversation with a human conversational partner and as such, a higher premium is placed on diversity.

To give a concrete example for the case of dialogue, the phrase ``I don’t know'' is usually a perfectly reasonable remark, and it appears quite often during normal human conversation. However, a chatbot that repeats ``I don’t know'' on every turn of dialogue makes for a very poor conversationalist. In such open-ended domains, being able to converse about a wide variety of topics with the occasional odd remark is highly preferred to merely repeating the safest possible remark over and over \citep{li2016diversity}.

% PRIMARY CONTRIBUTIONS HERE
% 1) Framework for evaluating diversity=entropy vs. quality=human judgments
% 2) We identify the likelihood trap
% 3) We investigate a method that tries to optimize 1) and 2).

To simultaneously capture both of these criteria, we propose framing the goal of generative language models as a multi-objective optimization over both \emph{quality} and \emph{diversity}. 
The proposed framework is flexible enough to encompass tasks that traditionally place low emphasis on diversity such as machine translation or summarization and others with high diversity such as storytelling.

Furthermore, the proposed framework enables us to evaluate existing decoding algorithms by comparing their performance along the entire quality-diversity spectrum. 
We compare a variety of commonly-used decoding algorithms in the first large-scale study of decoder quality, utilizing over 38,000 ratings on almost 10,000 samples.
We find that when diversity is highly valued, all decoders perform similarly, but when quality is viewed as more important, the recently proposed nucleus sampling \cite{holtzman2019curious} outperforms all other evaluated decoding algorithms.

Additionally, we use our framework to investigate the commonly held intuition that model likelihood is directly correlated with human quality judgments.
First, we explicitly test this belief by measuring the relationship between the quality of a sentence as judged by human raters and its likelihood under a generative model. 
Our findings confirm the existence of a \emph{likelihood trap}, the counter-intuitive observation that the highest likelihood sentences are of surprisingly low quality, despite a generally positive relationship between model likelihoods and human quality judgments. 
While this finding has been observed across a wide variety of language generation tasks and models ranging from news generation to machine translation \cite{cohen2018unconstrained,holtzman2019curious}, to our knowledge we are the first to explicitly quantify the relationship between the two at all points in the model probability space.

Secondly, we propose and evaluate selective sampling, \emph{selective sampling}, a decoder which emphasizes high probability sentences by drawing samples from the global temperature-adjusted model distribution.
While this has traditionally been considered intractable due to the difficulty of computing the partition function, we propose a procedure that uses rejection sampling to directly sample from the desired distribution without explicitly computing the partition function.
When evaluating this decoder alongside existing token-by-token decoders, we discover that it performs poorly even when taking the likelihood trap into account, suggesting that local token-by-token decoders may be capable of capturing structure that a global decoder does not.

\section{Framework}
\label{sec:framework}

In this section, we introduce a framework for trading off quality and diversity in language generation. Let $\mathcal{X}$ denote the space of all possible generated sentences. We consider autoregressive language models that decompose the likelihood of a sequence $(x_1, x_2, \ldots, x_n) = x_{1:n} \in \mathcal{X}$ token-by-token in a left-to-right fashion \cite{hamilton1994time, sutskever2014sequence}. Specifically, the (conditional) likelihood of the sequence is:

\begin{align}
    \pmd(x_{1:n} \mid c) = \prod_{i=1}^n \pmd(x_i | x_{1:i-1}, c)
    \label{eqn:temperature-sampling}
\end{align}

where $c$ is any additional conditioning signal, such as the previous turn of dialogue.
Random sampling is the decoding procedure that follows naturally from the factorization of the model's joint distribution where tokens are sampled one-at-a-time according to the model's conditional distribution, $\pmd(x_i | x_{1:i-1}, c)$.
Often $\pmd$ is not sampled from directly; it is first post-processed by a decoder to bias it toward already high-likelihood tokens.
%When unambiguous, we use the term ``model'' to both the language model itself and any strategy used to modify the output distribution before sampling occurs below.

In the proposed framework, we evaluate the quality of a single sentence $x\in\mathcal{X}$ by asking humans for a quality judgment $HJ(x)$.
We can define the quality $Q$ of a model as the expected human ``quality'' judgment for sentences drawn from it:
\begin{align*}
Q(p) = \E_{x \sim p} [HJ(x)]
\end{align*}
We measure the diversity of a model via the Shannon entropy $H$ \cite{shannon1948mathematical}, a diversity metric widely used across many fields beyond computer science including biology, economics, chemistry, and physics. Shannon entropy is given by:
\begin{align*}
H(p) = - \E_{x \sim p} [\log p(x)]    
\end{align*}
This allows us to define our multi-objective optimization problem as maximizing the following goal $G$:
\begin{align*}
G(p) = Q(p) + \lambda H(p)    
\end{align*}
where $\lambda$ is the task-specific measure of the relative importance of diversity and quality.
For open-ended tasks such as dialogue that place a premium on variety, decoder performance under large $\lambda$ is critical.
For more closed domain tasks such as summarization or machine translation, performance under smaller $\lambda$ (including possibly 0) is more important.

Ideally, one would optimize directly over goal $G$, but its dependence on human judgments makes direct optimization infeasible in practice. Instead, prior works optimize a proxy objective (such as the KL divergence) then employ a decoding algorithm to ``warp'' model $\pmd$ post-hoc towards higher values of $G$.

% Note that unlike frameworks which are based on minimizing divergences, nothing in the proposed framework rules out the possibility that a decoding algorithm has \emph{superhuman} performance.

In the following section, we relate our objective $G$ to existing decoders and investigate a novel decoding algorithm that normalized \emph{globally} across all possible sequences rather than simply token-by-token.

\section{Selective Sampling}
\label{likelihoodtrap}
\subsection{The Likelihood Trap}
% Neural language models tend to be worse at learning the tail of the distribution due to the relative shortage of training data for uncommon words.
Sequence likelihood is commonly used as a heuristic for selecting high-quality generations.
In the extreme, beam search approximates finding the \textit{single} most likely generation $x^* = \arg\max \log \pmd(x)$ and is the approach principally adopted in machine translation \cite{koehn2004pharaoh}.
% Evidence that overly-likely sequences aren't that great.

However, prior work has suggested that this assumption of a monotonically positive relationship between sequence likelihood and sequence quality breaks down at the extremes.
For example, it is well known in the machine translation and image captioning communities that after a certain point, increasing the beam size \emph{hurts} BLEU scores and other measures of quality \cite{stahlberg-byrne-2019-nmt, koehn2017six, vinyals2016show}. More recently \citet{holtzman2019curious} observe similar phenomena for open-ended generation where the highest likelihood sentences degenerate into extreme repetition.

We empirically quantify the relationship between sequence likelihoods and human quality judgments by sub-sampling a large number of context-response pairs representing a wide variety of model log likelihoods. We then request human crowdworkers to rate the quality of each response given the context on a five-point ``Terrible''-to-``High Quality'' scale.

\newcolumntype{b}{>{\hsize=1.5\hsize}X}
\newcolumntype{s}{>{\hsize=.5\hsize}X}

\begin{table*}[h!]
    \renewcommand{\arraystretch}{1.3}
    \small
    \label{tab:samples-table}
    \begin{tabularx}{\textwidth}{|s|b|c|c|}
        \hline
        \centering \textbf{Context} & \centering \textbf{Response} & \centering \bm{$\log p(x)$}& \textbf{Classification} \\
        \hline
        %\multirow{2}{*}{The Atlanta Falcons have started the 2015 season 4-0 under new head coach Dan Quinn. Quarterback Matt Ryan has the &  ... mental Tough O'Rourke Tough apology assessment category of virtue from Boser' Blog here. It's got letters and images on it and is utterly ...} & -177 & Nonsense \\\cline{2-4}
        %\rule{0pt}{4ex}  
        \vspace{0.1ex}
        \multirow{3}{*}{\shortstack[l]{The Atlanta Falcons \\ have started the 2015 \\
        season 4-0 under new \\
        head coach Dan Quinn. \\
        Quarterback Matt Ryan \\
        has the ...}} & ... mental \textcolor{blue}{Tough O'Rourke Tough apology assessment category of virtue from Boser' Blog here.} It's got letters and images on it and is utterly ... & -177 & \textcolor{blue}{Nonsense} \\\cline{2-4}
        & ... team afloat and looks closer to the 2010 Atlanta Falcons. Starting cornerback Desmond Trufant was one of the top players on the 2014 ... & -74 & Reasonable \\\cline{2-4}
        & ... team in the thick of the NFC South race. \textcolor{red}{The Atlanta Falcons have started the 2015 season 4-0 under new head coach Dan Quinn. Quarter}... & -14 & \textcolor{red}{Repetition}
        \\\cline{2-4}
        \Xhline{2\arrayrulewidth}
        \vspace{0.1ex} \multirow{3}{*}{\shortstack[l]{They have changed the \\
        phone menu to try to \\
        deflect us to email, \\
        but you can still get a \\
        live ...}} & ... answer from a female \textcolor{blue}{administratoria llallushoss@rahpx Sandra PJ Jenniea nightiopq HamidF daroyqg S')} ... & -229 & \textcolor{blue}{Nonsense} \\\cline{2-4}
        & ... message or call on line, so I suppose they are just using that as an excuse. Yet they are still telling people to change their telephone number... & -86 & Reasonable \\\cline{2-4}
        & ... link to a phone number here. \textcolor{red}{They have changed the phone menu to try to deflect us to email, but you can still get a live link to}... & -23 & \textcolor{red}{Repetition} \\
        \hline
    \end{tabularx}
    \caption{Examples of sentences at various model likelihoods. Sentences with very low $\log \pmd$ generate  \textcolor{blue}{nonsense}, while sentences that have high likelihood under the model often devolve into extreme \textcolor{red}{repetition}. Nonsense and repetition classifications shown here are only for illustrative purposes. Crowdworkers simply rated sentences for overall quality. See Appendix for more details.}
\end{table*}

Figure~\ref{fig:likelihood-trap} plots these ratings as a function of $\log \pmd$ and confirms that on average the highest quality generations are \emph{not} the most likely. Specifically, we find that response quality is generally positively related with $\log \pmd(x)$ up until an inflection point after which it becomes negatively related.
In our experiments, this inflection point occurs at $\log \pmd(x) = -58.09$.
Our findings suggest that while model likelihoods form a good proxy for response quality, naively maximizing over sentence likelihood leads to suboptimal response quality. We term this phenomenon the \emph{likelihood trap}.

Examples of the likelihood trap can be seen in Table~\ref{tab:samples-table}. Text sequences with extremely high likelihood tend to devolve into either extreme repetition or other nonsense, which some have attributed to either model biases \cite{holtzman2019curious} or aberrations in the training data \cite{ott2018analyzing}. We do not examine the underlying causes of the likelihood trap in this paper.

%Overall, we see an undeniable curve peaking 
%We use this value for the maximum likelihood cutoff $C$ in Figures~\ref{fig:entropy-vs-hj} and \ref{fig:sampling-method-vs-hj}.

%\iffalse
%\begin{figure}[h]
%    \centering
%    \includegraphics[width=0.8\columnwi%dth]{figures/likelihood-trap-with-fit.png}
%    \caption{Human judgment scores for sampled responses from GPT-2. Cutoff is selected by fitting a quadratic model and finding its maximum.}
%    \label{fig:likelihood-trap}
%\end{figure}
%\fi

%\fi

\subsection{Global Temperature Sampling}
Motivated by our findings that human judgments $HJ$ are positively correlated with model likelihoods for some interval of likelihoods, we investigate whether using $\log \pmd$ as a proxy for $HJ$ would lead to a better decoding algorithm. Specifically, we create a proxy quality function,
\begin{align*}
\hat{Q}(p) = 
\E_{x \sim p}
\left[\!\begin{aligned}
    & \log \pmd(x) ,    & \text{if }\log \pmd < \alpha.\\
    & -\infty,          & \text{otherwise}.
\end{aligned}\right]
\end{align*}
where $\alpha$ is selected as a hyperparameter.

Using globally-normalized temperature sampling, we can then approximate optimizing for $G$ through instead optimizing for the proxy objective $\hat{G}(p) = \hat{Q}(p) + H(p)$. This is due to the following proposition.

%In addition to commonly used token-by-token decoders, we also investigate a decoder that ignores local token-by-token probabilities and decides whether to upweight or downweight sequences at the global level rather token-by-token level. Global temperature sampling is motivated by the following proposition,

\begin{proposition}
Let $p$ be a probability distribution over some finite set $\mathcal{X}$. Let $H$ be the Shannon entropy function. The probability distribution $Q$ which minimizes the reverse KL Divergence $\infdiv{Q}{P}$ subject to $H(Q) = K$ for any achievable constant $K$ has the form,
\begin{align*}
    Q(x) = \frac{P(x)^{1/\tau}}{\sum_{x \in \mathcal{X}} P(x)^{1/\tau}}
\end{align*} 
for some temperature $\tau \in [0, 1]$.
\end{proposition}

\begin{proof}
Proof included in Appendix \ref{sec:globalproof}
\end{proof}

When applied to autoregressive models, global temperature sampling is usually dismissed as intractable due to the need to sum over the exponentially large space of all possible sequences in pursuit of the partition function $Z = \sum_x \pmd(x \mid c)^{1/\tau}$. Instead, past work typically decomposes sentences into tokens in a left-to-right autoregressive fashion and then use a local approximation,
\begin{align*}
     \hat{Z} = \prod_{i=1}^n \sum_{x_i} \pmd(x_i | x_{1:i-1}, c)^{\frac{1}{\tau}}
\end{align*}
where models are normalized locally over each set of tokens. This results in the well known (local) temperature sampling algorithm.

Unfortunately, while replacing the global partition function with a series of local ones transforms an exponential problem into a linear one, this approximation may bias the model towards favoring local structure over global structure. Indeed, we show via the following example that for some joint distributions, it is \textit{impossible} to represent globally-normalized temperature sampling via local temperature sampling, even if local temperature sampling is allowed to use a different temperature $\tau$ at each timestep.

\begin{proposition}
There exists a probability distribution $p$ and global temperature $\tau$ such that no choice of parameter allows local temperature sampling to match the joint distribution $p(x)^{1/\tau}$.
\end{proposition}

\begin{proof}
Figure 2 illustrates one such choice of $p$. By construction, local temperature sampling is forced to set $p_{\text{local}}(A) = p_{\text{local}}(B)$ regardless of the temperature hyperparameter used at that timestep. Setting a global temperature of $\tau = 0.5$ results in
\begin{align*}
&P(A) = \frac{0.1^2 + 0.4^2}{0.1^2 + 0.4^2 + 0.25^2 + 0.25^2} = 0.5763 \\
&P(B) = \frac{0.25^2 + 0.25^2}{0.1^2 + 0.4^2 + 0.25^2 + 0.25^2} = 0.4237
\end{align*}
which is not imitable by any local temperature setting.
\end{proof}

\begin{figure}[h]
    \centering
    \includegraphics[trim=160 100 160 0,clip,width=\columnwidth]{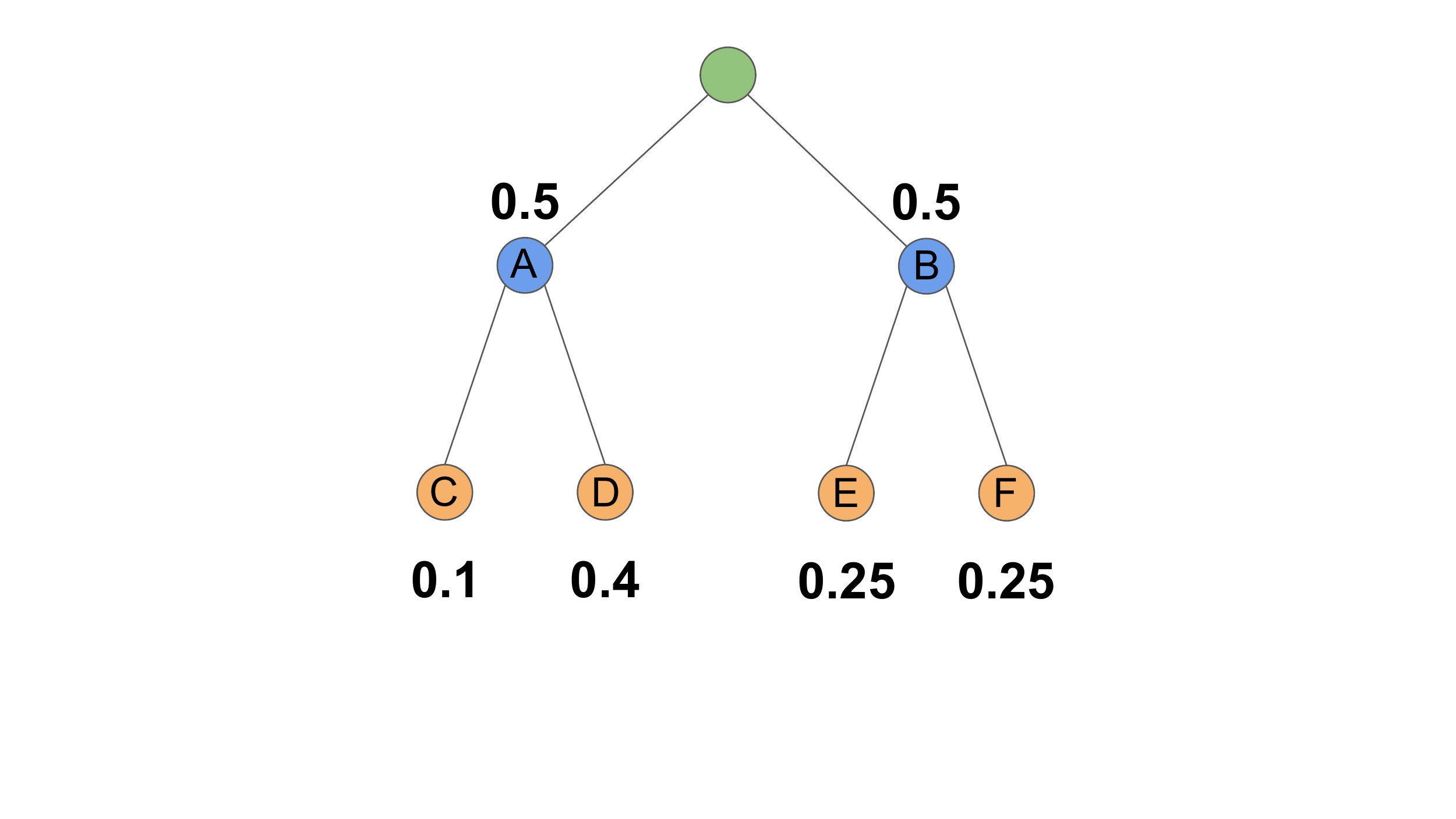}
    \caption{
    Any choice of temperature for local temperature sampling must have $P(A) = P(B)$. However, choosing global temperature $\tau = 0.5$ results in $P(A) = 0.5763$ and $P(B) = 0.4237$ which is impossible for any choice of local temperatures to satisfy.}
    \label{fig:impossibletree}
\end{figure}

% Talk about how global is hard.
% Segway into Approach section

Our core insight is that one can sample from the globally-normalized temperature sampling distribution without estimating the partition function $Z$ via rejection sampling. 
Rejection sampling \cite{forsythe1972neumann} gives an algorithm from sampling from an (unnormalized) energy distribution $p_{\text{energy}}$ if there exists a proposal distribution $q$ and constant $M$ such that $Mq \geq p_{\text{energy}}$.

We observe that $\pmd > \pmd^{\frac{1}{\tau}}$ for $\tau \in (0, 1)$ and $0 \leq p \leq 1$. This allows us to use $\pmd$ as the proposal distribution since the unnormalized probabilities of the global temperature sampling are given by $\pgb \propto \pmd^{\frac{1}{\tau}}$. %In consequence, we can use $\pmd(x)$ as a proposal distribution to rejection sample from the globally temperature annealed distribution $q_{\theta}(x)$.

\begin{algorithm}[h]
    \caption{Selective Sampling}
    \begin{algorithmic}
        \REQUIRE Global temperature $\tau$, Cutoff $\alpha$, and $\pmd$.
        \STATE Set $M = \alpha^{\frac{1}{\tau} - 1}$
        \WHILE{more sequences are required}
            \STATE Sample a sequence $x$ from $\pmd$.
            \IF{$\log \pmd(x) > \alpha$}
                \STATE Reject sample
            \ELSE
                \STATE Accept with probability $\frac{\pmd(x)^{\frac{1}{\tau} - 1}}{M}$
            \ENDIF
        \ENDWHILE
    \end{algorithmic}
    \label{alg:selective-sampling}
\end{algorithm}

% Our proposed solution. Show how to get a cutoff.

Selective sampling, by design, significantly increases the chances of sampling sequences with large values of $\log \pmd$. 
To avoid falling into the likelihood trap, we propose explicitly discarding generations $x$ where $\log \pmd(x)$ is greater than a chosen hyperparameter $\alpha$. 
An additional positive side effect of the cutoff is that the envelope constant $M$ can be chosen to create a tight bound on $p_{\text{energy}}$, which increases acceptance probabilities by several orders of magnitude.

A priori, it is not obvious how to choose $\alpha$ effectively. 
We propose collecting human judgments for a selection of random samples from $\pmd$ as illustrated in Figure~\ref{fig:likelihood-trap} and setting $\alpha$ equal to the discovered inflection point.
Note, that while this results in our procedure ignoring the set of sentences that individually have the highest probabilities, the total probability mass of this set is quite low: less than 0.5\% in our experiments.

\begin{figure}[h]
    \centering
    \includegraphics[width=0.9\columnwidth]{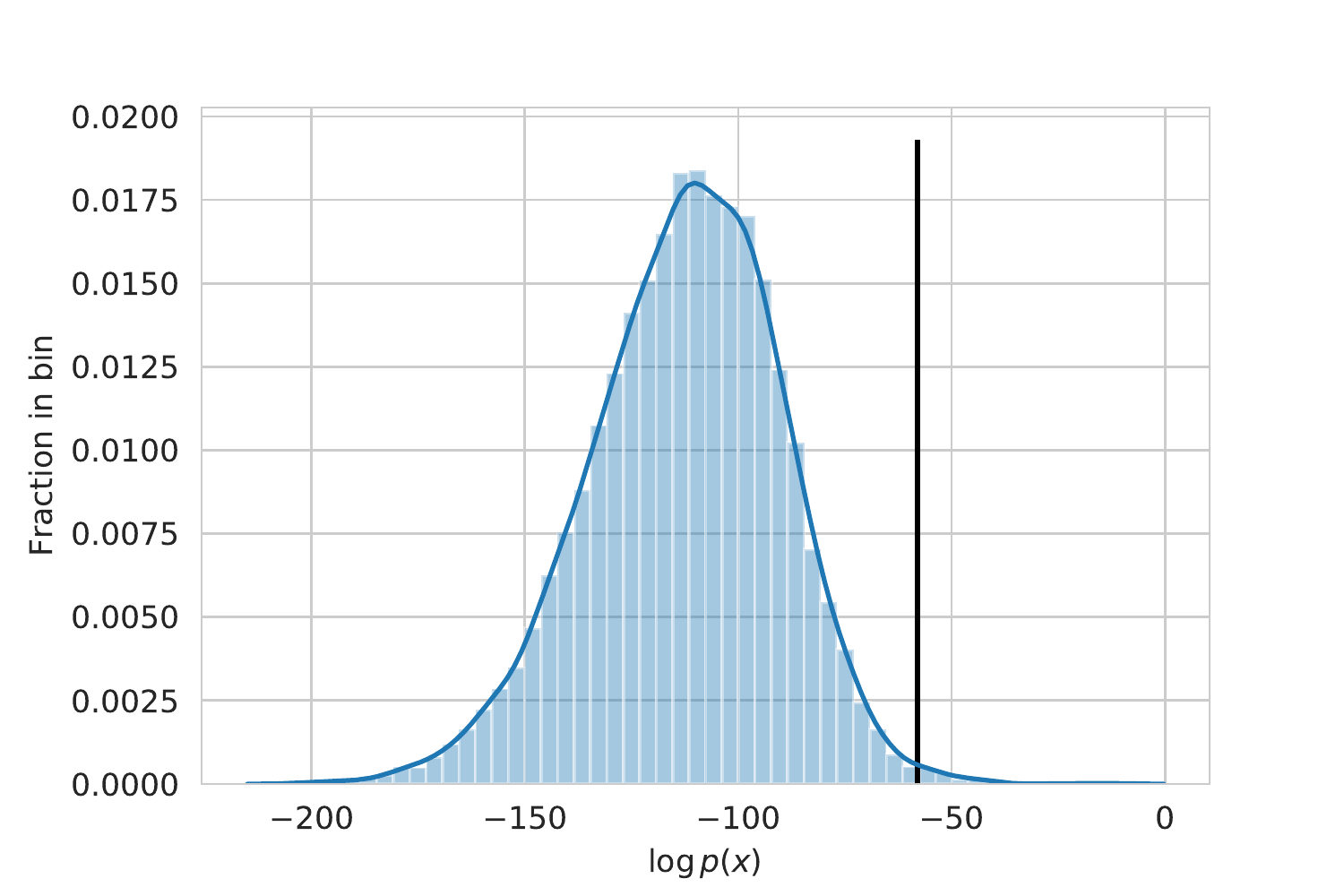}
    \caption{Histogram over $\pmd(x)$ for samples drawn from the same prompt. 99.5\% of samples have log likelihood less than the choosen cutoff $\alpha$ shown in black.}
    \label{fig:likelihood-trap-cutoff}
\end{figure}

% Draw samples according to p_{model}(x) as usual.
%
% Uses p_{model}(x) as an "umbrella". Reject samples that don't meet the rejection sampling criteria.
% Point at pseudocode here.
%
% This can mean more proposal samples per "real" sample.
%
% More compute required as temperature is decreased, but it's manageable.
% Show this in a plot.

%%%%%%%%%%%%%%%%%%%%%%%%%%%%%%%%%%%%%%%%%%%%%%%%%%%%%%%%%%%%%%%%%%%%%%%%%%%%%%%%%%%%%%%%%%%%%%%%%%%
\section{Experiments}
\label{sec:experiments}

% MAIN POINTS!!!
% 1) We have a framework for evaluating decoding algorithms in terms of diversity vs. quality
% --> random sampling is the worst.
% --> we validate that the quality/diversity tradeoff exists for top-k, nucleus, local temperature. 
% --> Quality is smooth wrt hparams.
% --> nucleus sampling is the best performing.
%
% 2) We observe a `likelihood trap', where samples with excessively high likelihood have lower quality.
% --> See log likelihood vs. HJ. See qualitative examples.
% --> See entropy vs. HJ, particularly nucleus=0.3.
%
% 3) We introduce "selective sampling" to optimize this graph. It fails to do so.
% --> For fixed entropy, it underperforms other methods.
% --> Why? A single max_logp cutoff is too coarse. See temperature vs. average log likelihood plot.
% --> Why? Choosing samples based on log p(x) alone is insufficient. See mimic plot.

In Section~\ref{sec:framework}, we introduce a theoretical framework for comparing decoding algorithms along a quality-diversity curve.
Under this framework, we evaluate several commonly used decoding algorithms in a human study described below.
In addition to selective sampling, we consider the following autoregressive decoding algorithms,
\begin{itemize}
    \item{\textbf{temperature}: Sample tokens with probability proportional to $\pmd(x_i | x_{1:i-1})^{1/t}$. $t$ varies from 0 to 1.}
    \item{\textbf{top-$k$} \cite{fan2018hierarchical}: Sample tokens only from the top-$k$ highest likelihood tokens in the vocabulary at each timestep. $k$ varies from 1 to vocabulary size.}
    \item{\textbf{top-$p$} (also known as nucleus sampling) \cite{holtzman2019curious}: Sample only from tokens comprising the top-$p$ percent of probability mass at each timestep, as ordered from most to least likely. $p$ varies from 0 to 1.}
\end{itemize}
At the extremes of their hyperparameter ranges, these algorithms all converge to greedy decoding and random sampling, respectively.
To sweep across the quality-diversity curve, we consider several hyperparameter settings per decoding algorithm below.
We refer to each decoding algorithm-hyperparameter combination as a `decoding configuration'.

%%%%%%%%%%%%% SETUP %%%%%%%%%%%%%
\subsection{Setup}

We apply each decoding algorithm to the 774M parameter variant of GPT-2 \cite{radford2019language}, a publicly-released language model.
To ground samples in a common context, we select a set of 48 examples from the GPT-2 test set to condition upon.
As samples are evaluated by human raters, we filter out examples containing explicit content or web markup.
Samples are drawn by conditioning on a `prompt' consisting of the first 20 space-delimited words of a test example.
As sample quality becomes ambiguous when samples are terse \cite{ippolito2019human}, we explicitly require all sampling methods to generate exactly 30 tokens, a length approximately equal to the prompt.

To estimate the expected Human judgment score $\E_{p}[HJ(x)]$ of the probability distributions induced by each decoding algorithm, we enlist a qualified pool of 146 Amazon Mechanical Turk (AMT) workers selected by satisfactory performance on a qualification task. 
Workers are presented sets of five samples, each conditioned on the same prompt and drawn from five different algorithm-hyperparameter configurations and asked to assign qualitative scores to each sample ranging from human-like to gibberish. The exact prompts, as shown to crowdworkers, are included in the Appendix.

Prior work has found that human annotaters have significant trouble in directly separating out machine and human generated responses when they are of similar quality, as the task of assessing sentence quality is highly subjective  \citep{ippolito2019human}. We found that constructing pairwise preference ratings by randomly pairing samples evaluated at the same time significantly reduced the variance of our results.
Specifically, if one sample is rated higher than the other, one is assigned a score of +1 and the other -1.
If both are rated equally, both are assigned a score of 0.
The score assigned to a decoding configuration is its average score across all pairwise preference ratings.
The average scores for each decoding strategy setting we experimented with are shown in Figure \ref{fig:mimic}.

%%%%%%%%%%%%% FRAMEWORK %%%%%%%%%%%%%
\subsection{Results}
We now introduce the first large-scale study comparing decoding algorithms and their hyperparameters.
Unlike all prior work \cite{holtzman2019curious,ippolito2019comparison}, we explicitly put decoding algorithms \emph{on equal footing} by comparing sample quality at equal points of diversity.
We consider five hyperparameter configurations per decoding algorithm for a total of twenty decoding algorithm-hyperparameter configurations.
For each configuration and prompt, we draw ten samples.
In total, workers rate nearly 10,000 samples resulting in over 38,000 paired ratings.
\begin{figure}[h]
    \centering
    \includegraphics[width=0.9\columnwidth]{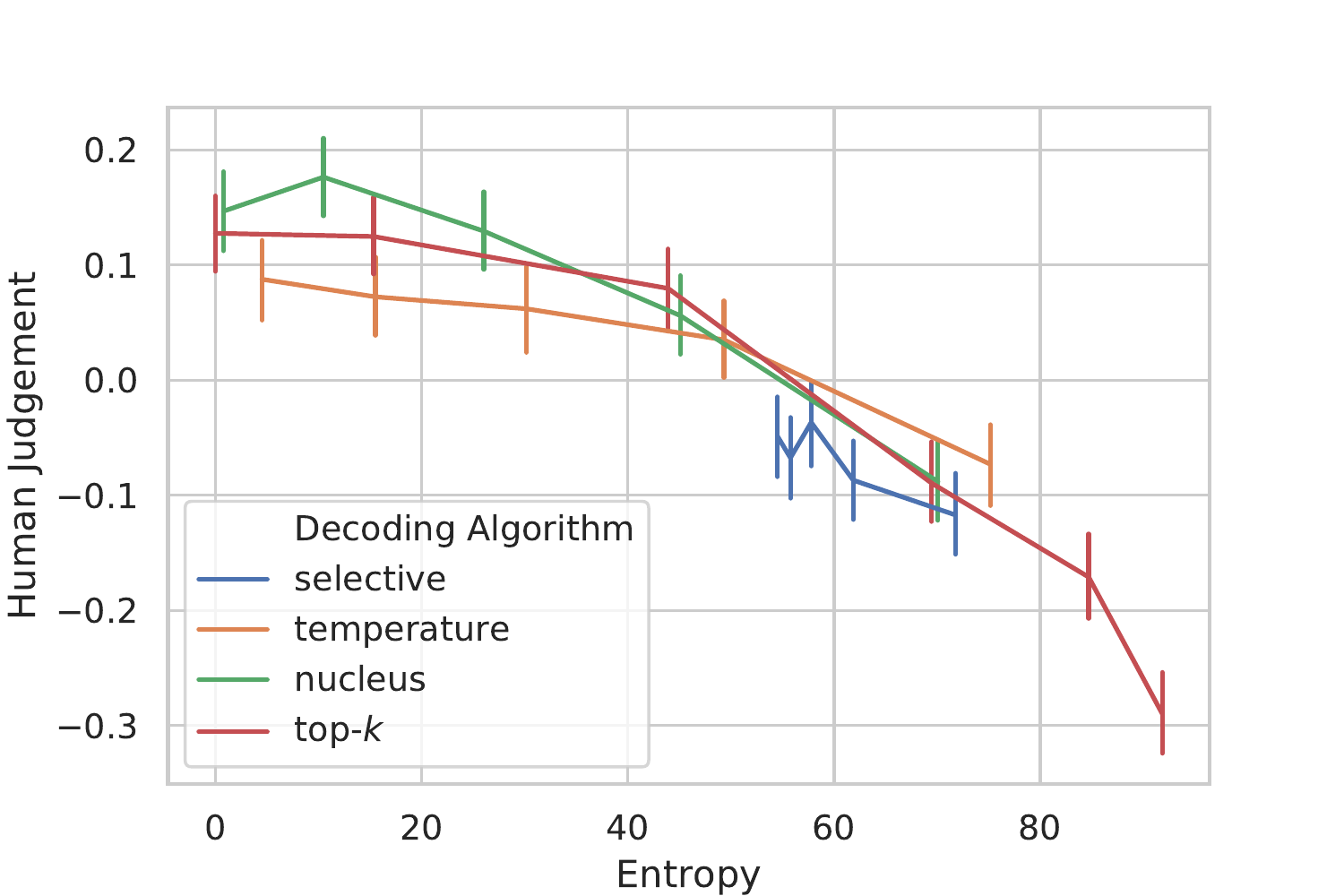}
    \caption{Human judgment scores as a function of decoding algorithm's entropy. Each point represents a single choice of decoding algorithm and hyperparameter. Error bars represent 95\% bootstrap confidence intervals.}
    \label{fig:entropy-vs-hj}
\end{figure}
\begin{figure*}[t]
    \centering
    \includegraphics[width=0.9\textwidth]{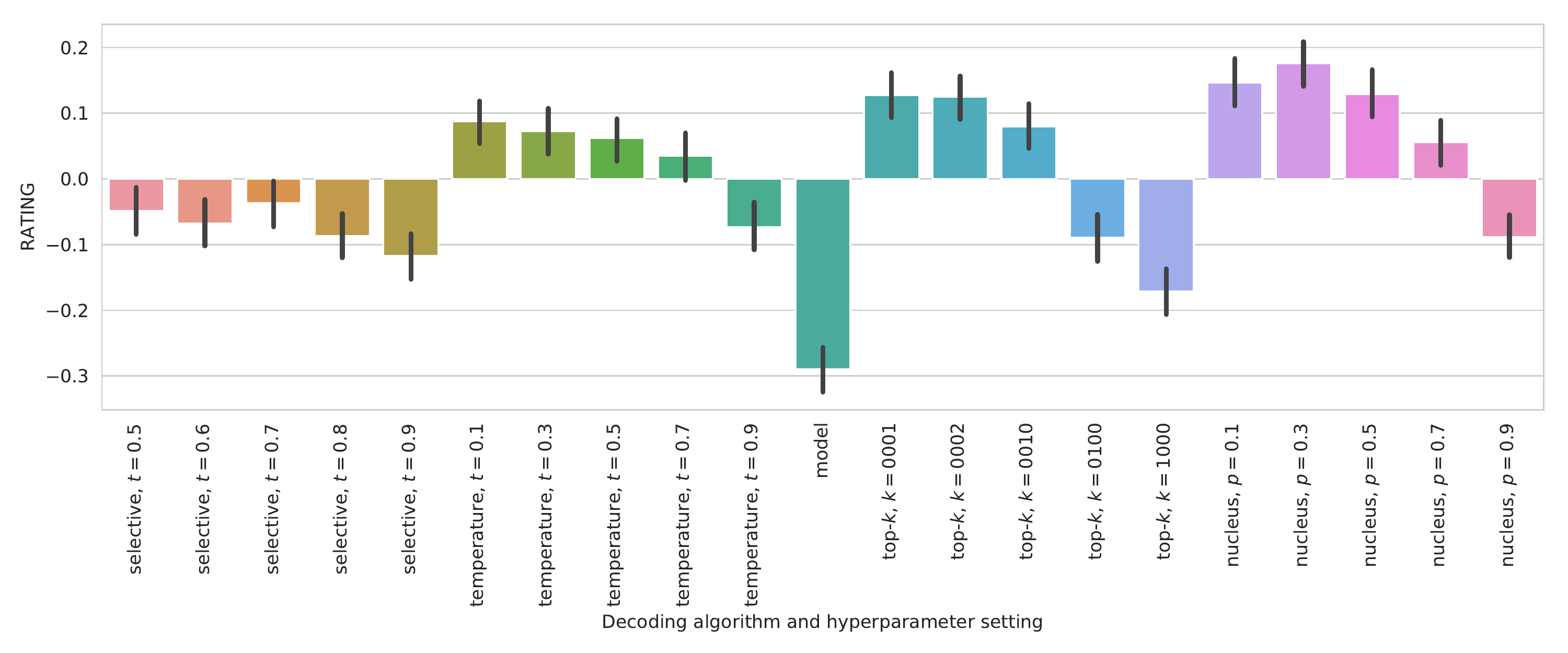}
    \caption{Human judgment scores for each decoding algorithm and hyperparameter choice. "Selective" is selective sampling and "model" is sampling directly from the probability distribution outputted by the language model. A score of 0 represents no preference. Selective sampling underperforms other more computationally efficient strategies.}
    \label{fig:sampling-method-vs-hj}
\end{figure*}

Our main results are summarized in Figures~\ref{fig:entropy-vs-hj} and \ref{fig:sampling-method-vs-hj}.
We empirically estimate the entropy of the probability distribution induced by each decoding configuration .
Reassuringly, both entropy and human judgment scores vary smoothly with decoding algorithm hyperparameter.

As expected, random sampling directly from the model $\pmd(x)$ is simultaneously the highest entropy \emph{and the lowest quality}.
This is empirically consistent with the long-standing intuition that decoding algorithms are critical to improving sample quality.
Why are samples from random sampling such poor quality?
Language models such as GPT-2 are trained to minimize the KL-divergence between a training set and the model distribution $\pmd$, an objective that prioritizes recall over precision \cite{arjovsky2017wasserstein}.
As a result, models tend to ensure that high quality sequences have high likelihood without insisting that all high likelihood sequences also have high quality.
When we evaluate samples from the model, we evaluate the latter condition.

Our second conclusion is that sample quality varies significantly with entropy for all decoding algorithms.
Moreover, when aligned on entropy, sample quality between all autoregressive decoding algorithms is comparable across a wide range.
It is only when entropy is low -- when decoding algorithms heavily influence sampling -- that sample quality between algorithms diverge.
In this regime, we find that nucleus sampling outperforms top-$k$, which in turn outperforms temperature sampling.
Observing such a difference should be unsurprising: the entropy of a distribution alone does not determine its sample quality.
We conclude that a fair comparison of decoding algorithms must not only compare at the same level of entropy but at a \emph{range} of entropy levels.

%%%%%%%%%%%%% LIKELIHOOD TRAP %%%%%%%%%%%%%

Finally and most surprisingly, we find that, in spite of its theoretical appeal, selective sampling consistently underperforms all other decoding algorithms considered.

\subsection{Selective Sampling}
Why does selective sampling underperform?
Our error analysis yields at least two potential causes: priors induced by decoding algorithms and a context-dependent likelihood trap.
\begin{figure}[h]
    \centering
    \includegraphics[width=0.7\columnwidth]{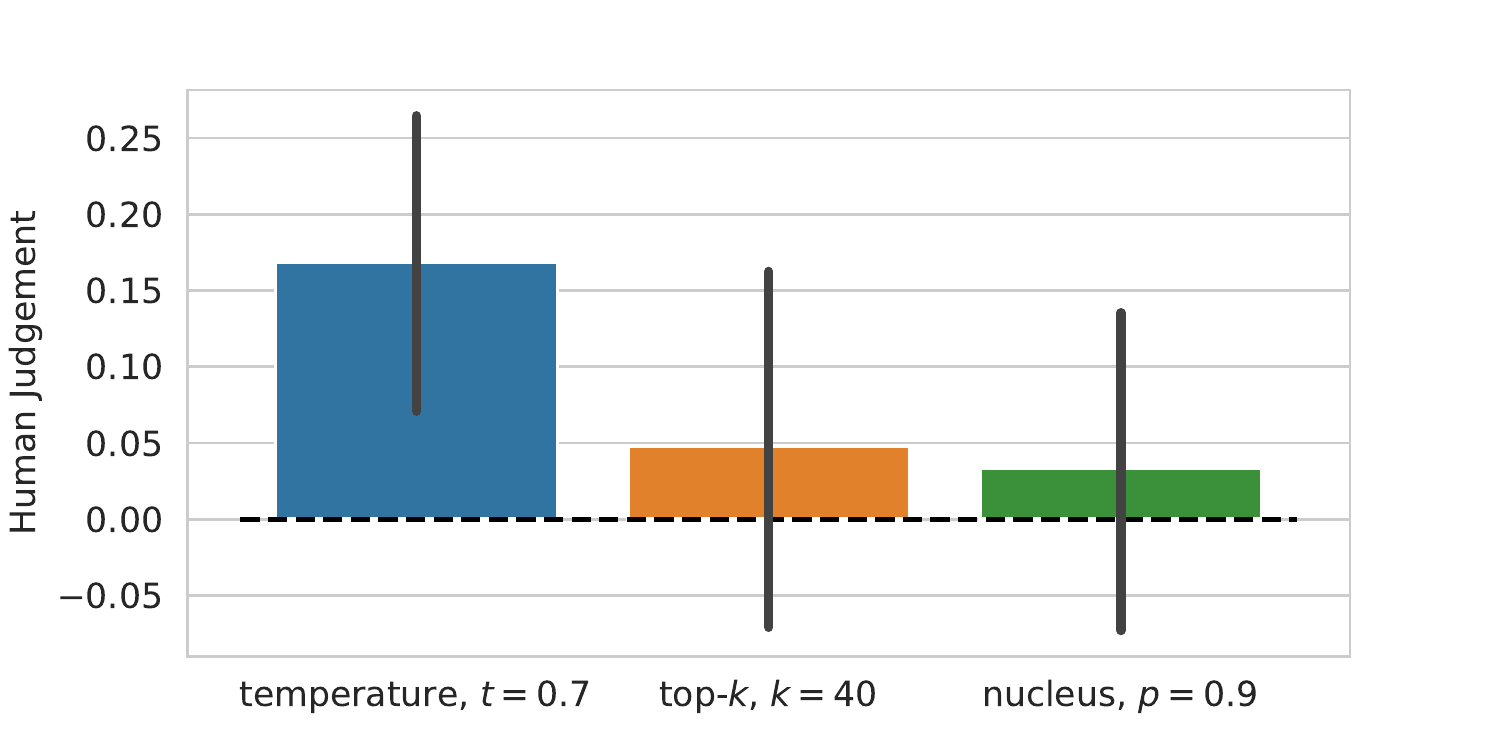}
    \caption{Human judgment scores for paired samples of equal log likelihood. 95\% bootstrap confidence intervals. Dotted line represents random judgments. For all decoding strategies it can be seen that decreasing diversity tends to lead to higher-judged outputs.}
    \label{fig:mimic}
\end{figure}
We first consider the implicit priors of autoregressive decoding algorithms.
Autoregressive decoding algorithms naturally favor sequences $x$ where each token $x_i$ has high model likelihood with respect to its conditional distribution $\pmd(x_i | x_{1:i-1})$.
Note that this is not necessarily the same as favoring \emph{all} high-likelihood sequences with high joint likelihood $\pmd(x)$; a criteria selective sampling targets at low temperatures.
We hypothesize that autoregressive decoding algorithms are inducing additional structure beyond high joint likelihood.

To test this hypothesis, we construct a human rating experiment that pairs random samples from a decoding algorithm with another random samples from the model distribution $\pmd$ such that the two samples have the \emph{same joint sentence likelihoods}.
In this way, we are able to control for differences in the distribution of $\pmd$ that different decoders induce and explicitly test only how various decoding algorithms promote different sequences with the same overall joint likelihood.
We draw samples from three commonly-used decoding configurations conditioned on all 48 prompts and compare each against random sampling by ask crowdworkers to rate which of the paired responses is of higher quality.

In Figure~\ref{fig:mimic}, we see that temperature sampling with $t=0.7$ is undeniably preferred to otherwise equivalent samples drawn directly from $\pmd$, though for other decoding configurations, the difference is currently less clear.
Selective sampling, a method with proposals drawn from $\pmd$, does not share this prior of its autoregressive locally normalized decoding counterparts.
We can thus conclude that the success of a decoding algorithm involves more than promoting high joint likelihood; in this way, selective sampling is deficient.
\begin{figure}[h]
    \centering
    \includegraphics[width=0.8\columnwidth]{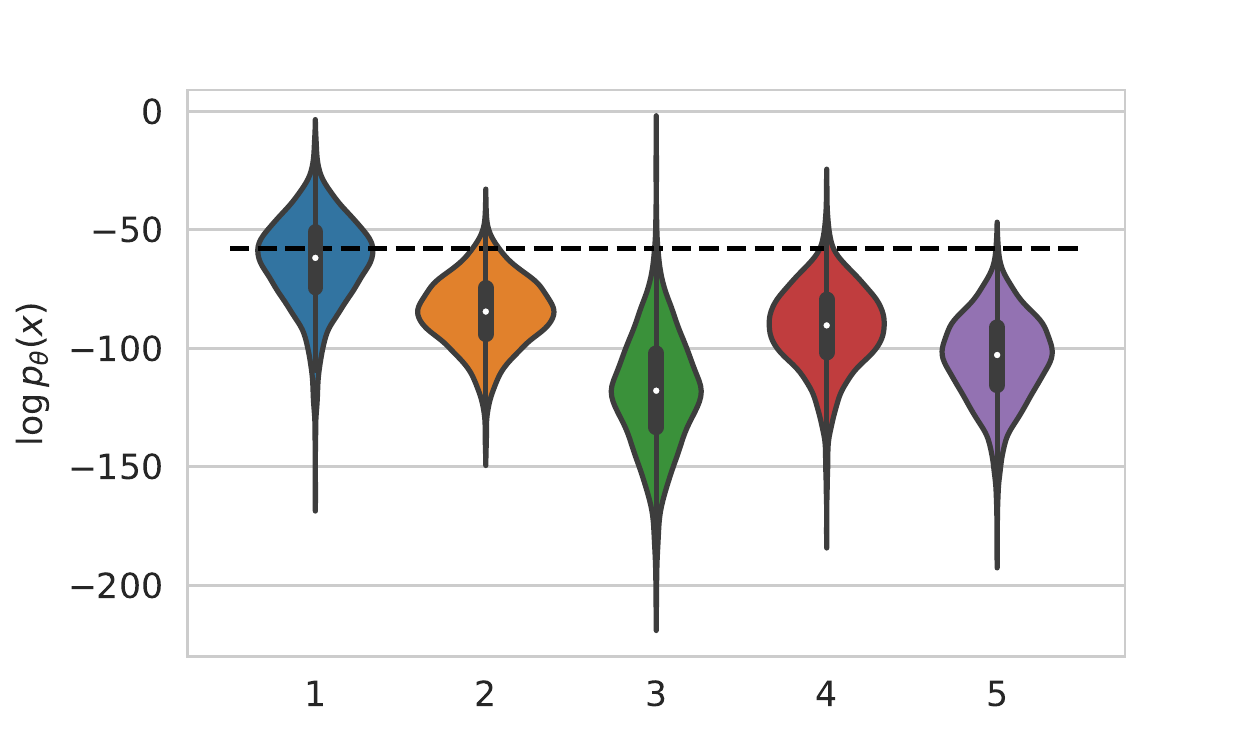}
    \caption{Distribution of prompt-conditional sample log likelihoods for five different prompts. The dotted-line represents the cutoff $\alpha$ used in experiments.}
    \label{fig:global-temp-logp-distribution}
\end{figure}

Second, we consider the distribution over sample log likelihoods conditioned on a fixed prompt as show in Figure~\ref{fig:global-temp-logp-distribution}
Depending on the prompt, the distribution over log likelihoods varies from prompt to prompt.
In selective sampling, we've elected to choose a single, global maximum likelihood constant $\alpha$.
For some prompts, this has nearly no impact -- nearly all samples have likelihood below the cutoff.
For others, this may eliminate nearly half of samples, leaving only those of lower quality. This suggests that a fixed cutoff $\alpha$ for all prompts may not be ideal.

Based on the prior experiments, we find that choice of decoding algorithm and its hyperparameter has a significant impact on sample quality and diversity.
Further, we find that sample quality and diversity can be traded for one another, and that the merit of a decoding algorithm requires comparing it to others at equivalent levels of diversity.
We also given evidence that autoregressive decoding algorithms induce additional preference beyond promoting samples with high joint likelihood; a beneficial preference selective sampling does not share.

%%%%%%%%%%%%%%%%%%%%%%%%%%%%%%%%%%%%%%%%%%%%%%%%%%%%%%%%%%%%%%%%%%%%%%%%%%%%%%%%%%%%%%%%%%%%%%%%%%%
\section{Related Work}
\textbf{Encouraging Diversity}\quad
Several recent work have proposed strategies for increasing the diversity of text generated by language models.
These approaches fall into two broad categories: (1) algorithmic changes to the decoding algorithm and (2) methods that involve training auxiliary language models or modifying the training paradigm for the main language model in some way.

The advantage of changing the decoding algorithm is that improvements can be rapidly be implemented on top of any already trained language model.
\citet{vijayakumar2016diverse},
\citet{li2016mutual},
\citet{tam2019clustered}, and
\citet{kulikov2018importance} all propose modifications to beam search to force it to explore a more diverse set of beams.
In contrast, modifications to random sampling that have been proposed aim to reduce diversity and thereby increase quality \citep{fan2018hierarchical,holtzman2019curious}.
\citet{ippolito2019comparison} compare many of these algorithmic advancements on the tasks of open-ended dialog and image captioning, concluding that the quality-diversity tradeoff makes it nearly impossible to say that any one of these methods is ubiquitously best.

We choose to evaluate three commonly used decoding methods: nucleus sampling \citep{holtzman2019curious}, top-k sampling \citep{fan2018hierarchical}, and temperature sampling. All three of these methods control the relative tradeoff between quality and diversity with a single hyperparameter. Top-k sampling samples from only the top-k most likely tokens at a timestep, proportionally according to the original probability. Nucleus sampling (also called top-p) sampling operates similarly, but chooses an adaptive $k$ such that the top-k tokens comprise of the top-p percent of the total probability mass at each timestep. Temperature sampling divides the logits of each token by the temperature hyperparameter before normalizing and converting the logits into sampling probabilities. 
In terms of diversity-promoting approaches that require training new language models, \cite{li2016diversity} use a language model that predicts the source sequence given the target sequence to rank candidate generations, penalizing generations that are too generic (have low $P(\text{source} \mid \text{target})$).
\citet{welleck2019neural} propose a novel loss function which discourages the model from assigning too high probability to repetitive wording. 
\citet{zhang2018generating} and \citet{xu2018diversity} use adversarial learning methods to encourage diversity.
Though these methods are promising, the extra complexity of training makes them less attractive for quickly improving upon existing language models.

The concept of oversampling generations and then ranking them has been popular since the days of statistical machine translation \citep{shen2004discriminative} but has also been used more recently in other domains \cite{li2016diversity,ippolito2019comparison,kriz2019complexity}.
Our particular contribution is to relate our sampling algorithm to the reverse KL divergence and competing objectives maximization.
We are also able to use this method to give approximate probability density estimates for sampled sentences, which typically cannot be done for algorithms that oversample generations.

\textbf{Likelihood Trap}\quad
We are far from the first to observe evidence of the likelihood trap. In particular, the machine translation and image captioning communities have long known that using higher beam sizes often leads to lower BLEU scores \citep{cohen2018unconstrained,vinyals2016show,yang2018breaking}.
In open-ended generation, \citet{holtzman2019curious} find similar results, observing that maximizing the likelihood generates extremely repetitive sentences. In addition to finding corroborating evidence that low quality generations appear at both the low and high probability extremes, our main contribution towards understanding the likelihood trap is the first explicit measurement of the relationship between model likelihoods and human quality judgments at all points in the model probability space, not just the endpoints.

\citet{ott2018analyzing} attempt to quantify the reasons behind the likelihood trap, proposing that the underlying issue is low quality examples in the training data.
They demonstrate that the likelihood trap can be avoided when restricting themselves to a significantly smaller dataset where each training point is carefully examined to guarantee that it is high quality.
However, given the recent interest in training increasingly large language models on increasing large datasets, it seems infeasible to guarantee the quality of every example included in the dataset.

%\todo{The following paragraph needs significant rewriting. See HUSE, Turing Test papers, representation learning...}

\textbf{Frameworks}\quad
Note that our framework is related, but not identical to many frameworks such as  \citet{hashimoto2019unifying, kingma2013auto, goodfellow2014gan} which ask that generative models mimic the training distribution exactly. 
While some tasks do require indistinguishability as the ultimate goal (e.g. representation learning \cite{bengio2013representation}, Turing Test \cite{turing2009computing, ippolito2019human}, etc.), this is typically not the case for most generation tasks. 
Humans make errors, but a ``perfect'' model would not seek to imitate these mistakes.
Because we ground quality evaluations in human judgments rather than on any statistical measure, our framework is easily able to capture the possibility of superhuman performance in ways that frameworks based solely on a statistical divergence would find difficult.

\section{Conclusion}
In this paper, we propose a framework for credibly evaluating decoding algorithms and use it to conduct the first large scale evaluation of decoding algorithms by measuring their performance along the entire quality-diversity frontier. 
Our findings suggest that existing decoding algorithms are more or less interchangeable in high diversity settings, but that nucleus sampling performs best when quality is valued over diversity.
Additionally, we provide evidence for the existence of a \emph{likelihood trap} and are the first to explicitly measure the relationship between $\log \pmd$ and human judgments. Finally, we propose and evaluate selective sampling, the first algorithm that can tractably estimate \emph{globally} normalized temperature sampling.

In the future, we hope to extend our work to additional generative language models as well as other modalities such as image and music generation. Additionally, we leave questions of whether selective sampling can be improved via choice of an adaptive cutoff that can vary based on the prompt or proposal distributions other than random sampling for future discovery.

%FOR EACH MODEL, estimating the optimal Pareto frontier it can generate. Using a "better" model will result in a more accurate estimate of the human quality judgments for each sentence, resulting in a better overall Pareto frontier.

%However, since our framework relies on 

%it is not directly suitable for evaluating generative tasks in the same vein that other work with the quality diversity tradeoff is [CITE].

%Hitherto, estimating the quality-diversity tradeoff of a model has been used only for evaluation purposes; the point however, is to improve them. In this work, we propose a novel framework casting generation as a multi-objective problem simultaneously optimizing for both quality and diversity of the generated outputs. Our specific implementation of this framework, selective sampling, uses truncated log likelihood and Shannon entropy as the quality and diversity proxies respectively, and is able to reach areas of the Pareto frontier that no other existing decoding method can reach.

%Unlike many other frameworks (HUSE) based purely on a divergence between distributions, our framework is easily extendable to allow superhuman performance simply by achieving any point that dominates the original distribution on the quality - diversity pareto optimal curve.

%Our multi-objective framework is suitable only where generation is the end goal. If the goal is instead to say, learn representations for a downstream task, you want a divergence based framework.

\bibliography{paper}

\begin{thebibliography}{37}
\providecommand{\natexlab}[1]{#1}
\providecommand{\url}[1]{\texttt{#1}}
\expandafter\ifx\csname urlstyle\endcsname\relax
  \providecommand{\doi}[1]{doi: #1}\else
  \providecommand{\doi}{doi: \begingroup \urlstyle{rm}\Url}\fi

\bibitem[Adiwardana et~al.(2020)Adiwardana, Luong, So, Hall, Fiedel, Thoppilan,
  Yang, Kulshreshtha, Nemade, Lu, et~al.]{adiwardana2020towards}
Adiwardana, D., Luong, M.-T., So, D.~R., Hall, J., Fiedel, N., Thoppilan, R.,
  Yang, Z., Kulshreshtha, A., Nemade, G., Lu, Y., et~al.
\newblock Towards a human-like open-domain chatbot.
\newblock \emph{arXiv preprint arXiv:2001.09977}, 2020.

\bibitem[Arjovsky et~al.(2017)Arjovsky, Chintala, and
  Bottou]{arjovsky2017wasserstein}
Arjovsky, M., Chintala, S., and Bottou, L.
\newblock Wasserstein gan.
\newblock \emph{arXiv preprint arXiv:1701.07875}, 2017.

\bibitem[Bengio et~al.(2013)Bengio, Courville, and
  Vincent]{bengio2013representation}
Bengio, Y., Courville, A., and Vincent, P.
\newblock Representation learning: A review and new perspectives.
\newblock \emph{IEEE transactions on pattern analysis and machine
  intelligence}, 35\penalty0 (8):\penalty0 1798--1828, 2013.

\bibitem[Cohen \& Beck(2018)Cohen and Beck]{cohen2018unconstrained}
Cohen, E. and Beck, J.~C.
\newblock (unconstrained) beam search is sensitive to large search
  discrepancies.
\newblock 2018.

\bibitem[de~Masson~d'Autume et~al.(2019)de~Masson~d'Autume, Mohamed, Rosca, and
  Rae]{de2019training}
de~Masson~d'Autume, C., Mohamed, S., Rosca, M., and Rae, J.
\newblock Training language gans from scratch.
\newblock In \emph{Advances in Neural Information Processing Systems}, pp.\
  4302--4313, 2019.

\bibitem[Fan et~al.(2018)Fan, Lewis, and Dauphin]{fan2018hierarchical}
Fan, A., Lewis, M., and Dauphin, Y.
\newblock Hierarchical neural story generation.
\newblock \emph{arXiv preprint arXiv:1805.04833}, 2018.

\bibitem[Forsythe(1972)]{forsythe1972neumann}
Forsythe, G.~E.
\newblock Von neumann’s comparison method for random sampling from the normal
  and other distributions.
\newblock \emph{Mathematics of Computation}, 26\penalty0 (120):\penalty0
  817--826, 1972.

\bibitem[Gehring et~al.(2017)Gehring, Auli, Grangier, Yarats, and
  Dauphin]{gehring2017convolutional}
Gehring, J., Auli, M., Grangier, D., Yarats, D., and Dauphin, Y.~N.
\newblock Convolutional sequence to sequence learning.
\newblock In \emph{Proceedings of the 34th International Conference on Machine
  Learning-Volume 70}, pp.\  1243--1252. JMLR. org, 2017.

\bibitem[Goodfellow et~al.(2014)Goodfellow, Pouget-Abadie, Mirza, Xu,
  Warde-Farley, Ozair, Courville, and Bengio]{goodfellow2014gan}
Goodfellow, I.~J., Pouget-Abadie, J., Mirza, M., Xu, B., Warde-Farley, D.,
  Ozair, S., Courville, A., and Bengio, Y.
\newblock Generative adversarial nets.
\newblock In \emph{Advances in Neural Information Processing Systems
  (NeurIPS)}, 2014.

\bibitem[Hamilton(1994)]{hamilton1994time}
Hamilton, J.~D.
\newblock \emph{Time series analysis}, volume~2.
\newblock Princeton New Jersey, 1994.

\bibitem[Hashimoto et~al.(2019)Hashimoto, Zhang, and
  Liang]{hashimoto2019unifying}
Hashimoto, T.~B., Zhang, H., and Liang, P.
\newblock Unifying human and statistical evaluation for natural language
  generation.
\newblock \emph{arXiv preprint arXiv:1904.02792}, 2019.

\bibitem[Holtzman et~al.(2019)Holtzman, Buys, Forbes, and
  Choi]{holtzman2019curious}
Holtzman, A., Buys, J., Forbes, M., and Choi, Y.
\newblock The curious case of neural text degeneration.
\newblock \emph{arXiv preprint arXiv:1904.09751}, 2019.

\bibitem[Ippolito et~al.(2019{\natexlab{a}})Ippolito, Duckworth,
  Callison-Burch, and Eck]{ippolito2019human}
Ippolito, D., Duckworth, D., Callison-Burch, C., and Eck, D.
\newblock Human and automatic detection of generated text.
\newblock \emph{arXiv preprint arXiv:1911.00650}, 2019{\natexlab{a}}.

\bibitem[Ippolito et~al.(2019{\natexlab{b}})Ippolito, Kriz, Kustikova, Sedoc,
  and Callison-Burch]{ippolito2019comparison}
Ippolito, D., Kriz, R., Kustikova, M., Sedoc, J., and Callison-Burch, C.
\newblock Comparison of diverse decoding methods from conditional language
  models.
\newblock \emph{arXiv preprint arXiv:1906.06362}, 2019{\natexlab{b}}.

\bibitem[Kingma \& Welling(2013)Kingma and Welling]{kingma2013auto}
Kingma, D.~P. and Welling, M.
\newblock Auto-encoding variational bayes.
\newblock \emph{arXiv preprint arXiv:1312.6114}, 2013.

\bibitem[Koehn(2004)]{koehn2004pharaoh}
Koehn, P.
\newblock Pharaoh: a beam search decoder for phrase-based statistical machine
  translation models.
\newblock In \emph{Conference of the Association for Machine Translation in the
  Americas}, pp.\  115--124. Springer, 2004.

\bibitem[Koehn \& Knowles(2017)Koehn and Knowles]{koehn2017six}
Koehn, P. and Knowles, R.
\newblock Six challenges for neural machine translation.
\newblock \emph{arXiv preprint arXiv:1706.03872}, 2017.

\bibitem[Kriz et~al.(2019)Kriz, Sedoc, Apidianaki, Zheng, Kumar, Miltsakaki,
  and Callison-Burch]{kriz2019complexity}
Kriz, R., Sedoc, J., Apidianaki, M., Zheng, C., Kumar, G., Miltsakaki, E., and
  Callison-Burch, C.
\newblock Complexity-weighted loss and diverse reranking for sentence
  simplification.
\newblock \emph{arXiv preprint arXiv:1904.02767}, 2019.

\bibitem[Kulikov et~al.(2018)Kulikov, Miller, Cho, and
  Weston]{kulikov2018importance}
Kulikov, I., Miller, A.~H., Cho, K., and Weston, J.
\newblock Importance of a search strategy in neural dialogue modelling.
\newblock 2018.

\bibitem[Li \& Jurafsky(2016)Li and Jurafsky]{li2016mutual}
Li, J. and Jurafsky, D.
\newblock Mutual information and diverse decoding improve neural machine
  translation.
\newblock 2016.

\bibitem[Li et~al.(2016)Li, Galley, Brockett, Gao, and Dolan]{li2016diversity}
Li, J., Galley, M., Brockett, C., Gao, J., and Dolan, W.~B.
\newblock A diversity-promoting objective function for neural conversation
  models.
\newblock pp.\  110--119, 2016.

\bibitem[Ott et~al.(2018)Ott, Auli, Grangier, and Ranzato]{ott2018analyzing}
Ott, M., Auli, M., Grangier, D., and Ranzato, M.
\newblock Analyzing uncertainty in neural machine translation.
\newblock \emph{arXiv preprint arXiv:1803.00047}, 2018.

\bibitem[Radford et~al.(2019)Radford, Wu, Child, Luan, Amodei, and
  Sutskever]{radford2019language}
Radford, A., Wu, J., Child, R., Luan, D., Amodei, D., and Sutskever, I.
\newblock Language models are unsupervised multitask learners.
\newblock \emph{OpenAI Blog}, 1\penalty0 (8):\penalty0 9, 2019.

\bibitem[Shannon(1948)]{shannon1948mathematical}
Shannon, C.~E.
\newblock A mathematical theory of communication.
\newblock \emph{Bell system technical journal}, 27\penalty0 (3):\penalty0
  379--423, 1948.

\bibitem[Shen et~al.(2004)Shen, Sarkar, and Och]{shen2004discriminative}
Shen, L., Sarkar, A., and Och, F.~J.
\newblock Discriminative reranking for machine translation.
\newblock pp.\  177--184, 2004.

\bibitem[Stahlberg \& Byrne(2019)Stahlberg and Byrne]{stahlberg-byrne-2019-nmt}
Stahlberg, F. and Byrne, B.
\newblock On {NMT} search errors and model errors: Cat got your tongue?
\newblock In \emph{Proceedings of the 2019 Conference on Empirical Methods in
  Natural Language Processing and the 9th International Joint Conference on
  Natural Language Processing (EMNLP-IJCNLP)}, pp.\  3356--3362, Hong Kong,
  China, November 2019. Association for Computational Linguistics.
\newblock \doi{10.18653/v1/D19-1331}.
\newblock URL \url{https://www.aclweb.org/anthology/D19-1331}.

\bibitem[Sutskever et~al.(2014)Sutskever, Vinyals, and
  Le]{sutskever2014sequence}
Sutskever, I., Vinyals, O., and Le, Q.~V.
\newblock Sequence to sequence learning with neural networks.
\newblock In \emph{Advances in neural information processing systems}, pp.\
  3104--3112, 2014.

\bibitem[Tam et~al.(2019)Tam, Ding, Niu, and Zhou]{tam2019clustered}
Tam, Y.-C., Ding, J., Niu, C., and Zhou, J.
\newblock Cluster-based beam search for pointer-generator chatbot grounded by
  knowledge.
\newblock In \emph{Dialog System Technology Challenges 7 at AAAI 2019}, 2019.

\bibitem[Turing(2009)]{turing2009computing}
Turing, A.~M.
\newblock Computing machinery and intelligence.
\newblock In \emph{Parsing the Turing Test}, pp.\  23--65. Springer, 2009.

\bibitem[Vaswani et~al.(2017)Vaswani, Shazeer, Parmar, Uszkoreit, Jones, Gomez,
  Kaiser, and Polosukhin]{vaswani2017attention}
Vaswani, A., Shazeer, N., Parmar, N., Uszkoreit, J., Jones, L., Gomez, A.~N.,
  Kaiser, {\L}., and Polosukhin, I.
\newblock Attention is all you need.
\newblock In \emph{Advances in neural information processing systems}, pp.\
  5998--6008, 2017.

\bibitem[Vijayakumar et~al.(2016)Vijayakumar, Cogswell, Selvaraju, Sun, Lee,
  Crandall, and Batra]{vijayakumar2016diverse}
Vijayakumar, A.~K., Cogswell, M., Selvaraju, R.~R., Sun, Q., Lee, S., Crandall,
  D., and Batra, D.
\newblock Diverse beam search: Decoding diverse solutions from neural sequence
  models.
\newblock 2016.

\bibitem[Vinyals et~al.(2016)Vinyals, Toshev, Bengio, and
  Erhan]{vinyals2016show}
Vinyals, O., Toshev, A., Bengio, S., and Erhan, D.
\newblock Show and tell: Lessons learned from the 2015 mscoco image captioning
  challenge.
\newblock \emph{IEEE transactions on pattern analysis and machine
  intelligence}, 39\penalty0 (4):\penalty0 652--663, 2016.

\bibitem[Welleck et~al.(2019)Welleck, Kulikov, Roller, Dinan, Cho, and
  Weston]{welleck2019neural}
Welleck, S., Kulikov, I., Roller, S., Dinan, E., Cho, K., and Weston, J.
\newblock Neural text generation with unlikelihood training.
\newblock \emph{arXiv preprint arXiv:1908.04319}, 2019.

\bibitem[Xu et~al.(2018)Xu, Ren, Lin, and Sun]{xu2018diversity}
Xu, J., Ren, X., Lin, J., and Sun, X.
\newblock Diversity-promoting gan: A cross-entropy based generative adversarial
  network for diversified text generation.
\newblock In \emph{Proceedings of the 2018 Conference on Empirical Methods in
  Natural Language Processing}, pp.\  3940--3949, 2018.
\newblock URL \url{https://www.aclweb.org/anthology/D18-1428}.

\bibitem[Yang et~al.(2018)Yang, Huang, and Ma]{yang2018breaking}
Yang, Y., Huang, L., and Ma, M.
\newblock Breaking the beam search curse: A study of (re-) scoring methods and
  stopping criteria for neural machine translation.
\newblock \emph{arXiv preprint arXiv:1808.09582}, 2018.

\bibitem[Zhang et~al.(2018)Zhang, Galley, Gao, Gan, Li, Brockett, and
  Dolan]{zhang2018generating}
Zhang, Y., Galley, M., Gao, J., Gan, Z., Li, X., Brockett, C., and Dolan, B.
\newblock Generating informative and diverse conversational responses via
  adversarial information maximization.
\newblock In \emph{Advances in Neural Information Processing Systems}, pp.\
  1815--1825, 2018.

\bibitem[Zhang et~al.(2019)Zhang, Sun, Galley, Chen, Brockett, Gao, Gao, Liu,
  and Dolan]{zhang2019dialogpt}
Zhang, Y., Sun, S., Galley, M., Chen, Y.-C., Brockett, C., Gao, X., Gao, J.,
  Liu, J., and Dolan, B.
\newblock Dialogpt: Large-scale generative pre-training for conversational
  response generation.
\newblock 2019.

\end{thebibliography}
\bibliographystyle{icml2020}
\clearpage
\appendix
\section{Appendix}
\label{sec:appendix}

\subsection{Proof of Proposition 1}
\label{sec:globalproof}

Notice first that, subject to $H(Q) = K$,
\[
 \arg\min_Q \infdiv{Q}{P} = \arg\max_Q \sum_{x \in X} Q(x) \log P(x)
\]

Properly choosing $K^*$ allows us to write the Lagrangian dual for the above constrained optimization problem as

\begin{align}
L(Q, \lambda, \mu) & = \lambda (\sum_{x\in \mathcal{X}} Q(x) \log P(x)) - K^*) \\
& + H(Q) + \mu ((\sum_{x\in \mathcal{X}} Q(X)) - 1) = 0
\end{align}

For any $x \in \mathcal{X}$
\begin{align*}
& \nabla_{Q(x)} L(Q, \lambda, \mu) \\
= \: & \lambda \log P(x) + \log Q(x) + 1 + \mu = 0 \\
\Rightarrow \; & Q(x) = \frac{P(x)^{-\lambda}}{e^{1 + \mu}}
\end{align*}

Setting $\lambda = -\frac{1}{\tau}$ and $\mu = -1 + \log \sum_{x \in \mathcal{X}} P(x)^{\frac{1}{\tau}}$
immediately gives us temperature sampling. Finally, observing that positive temperatures give us the local maxima and negative temperatures give us the local minima completes the proof.

\subsection{Experimental Design}
\label{sec:experimental-design}

In this section, we describe the design of experiments presented in Section~\ref{sec:experiments} in greater detail.

We begin by describing the task presented to crowdsourced raters.
A sample task is shown in Figure~\ref{fig:example-task}.
%   [FIGURE]: Show a sample AMT task.
Each task consists of a ``context'' sequence of the first 20 words in a news article.\footnote{News articles are sourced from GPT-2's WebText dataset. https://github.com/openai/gpt-2-output-dataset}
We then present the rater with five continuations of 30 word-piece tokens.
The rater assigns a label of ``High Quality'', ``Decent'', ``Passable'', ``Bad'' or ``Terrible'' to each.
We note that these labels are inherently subjective, and include a description and reference example before each task to calibrate the rater.
The same description and example is repeated in Figure~\ref{fig:task-instructions}.
%   [FIGURE]: Show rater instructions.

\begin{figure*}[h]
    \centering
    \includegraphics[width=\textwidth]{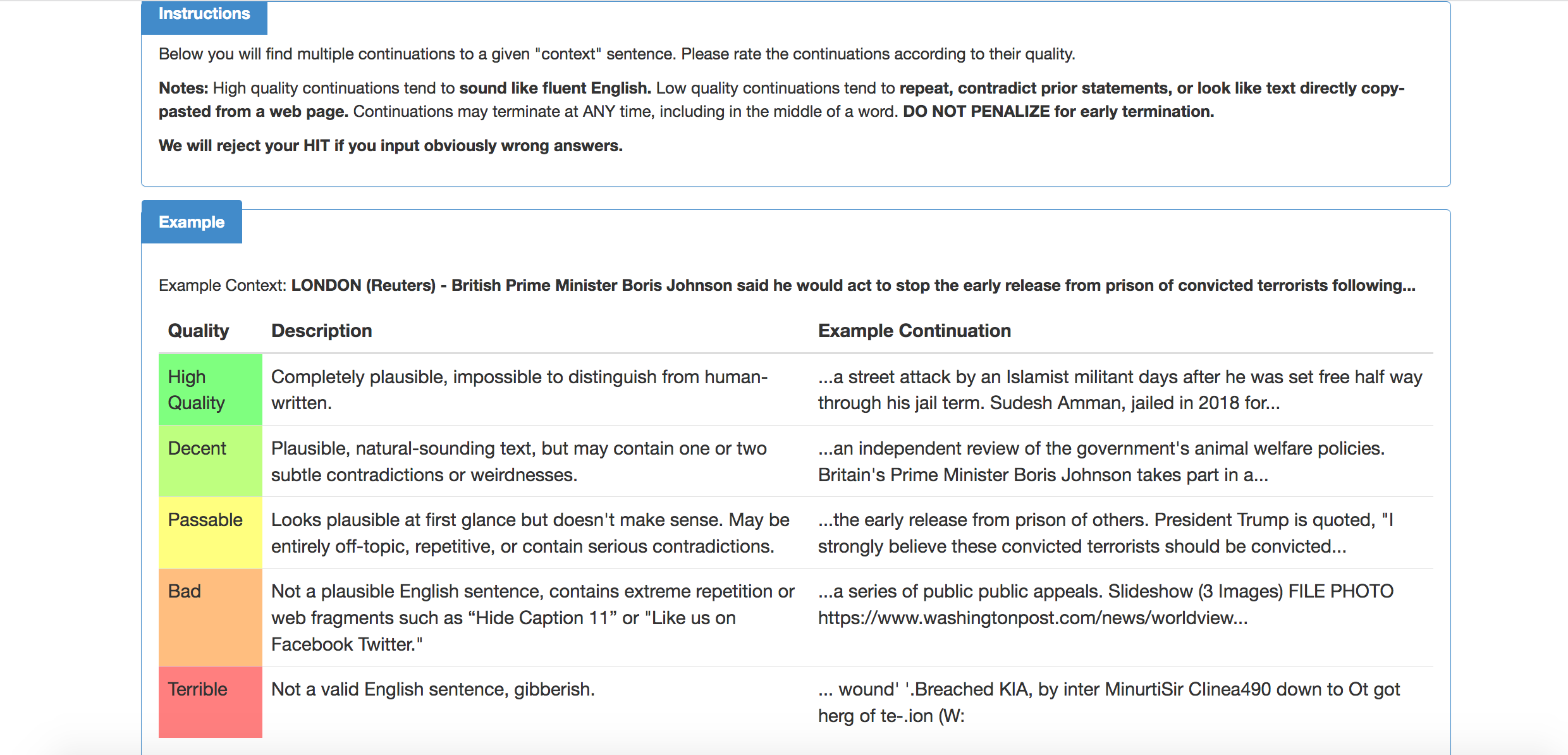}
    \caption{Instructions for the crowdworker task. Each sentence continuation is labeled on a scale from ``Terrible'' to ``High Quality''. A description of each label and an example continuation that fits each each is provided before each task.}
    \label{fig:task-instructions}
\end{figure*}

\begin{figure*}[h]
    \centering
    \includegraphics[width=\textwidth]{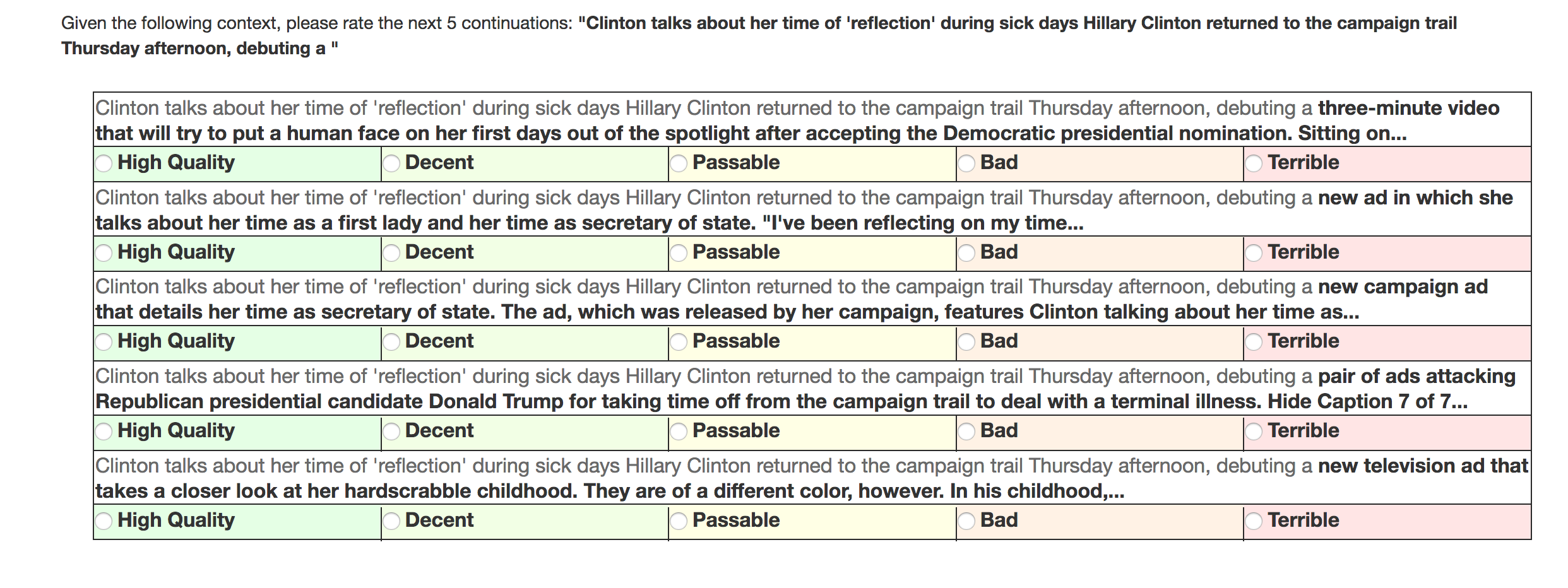}
    \caption{Sample crowdworker task used for the main evaluation results. Raters assign a label on a scale from ``Terrible'' to ``High Quality'' to each of five continuations sharing a common context of twenty words. Each continuation is generated by a different sampling method and hyperparameter.}
    \label{fig:example-task}
\end{figure*}

In preliminary experiments, we found examples and instructions insufficient for achieving repeatable results.
Manual inspection of rater responses revealed a failure to interpret the labels correctly as well as spammers who would always choose the same response for every prompt.
As a result, we crafted a qualification exam of five continuations.
Only raters which rated all five continuations correctly or nearly correctly\footnote{Raters which incorrectly labeled at most one continuation with a label at most one level off (e.g. if the correct answer is "Bad", acceptable errors are "Passable" and "Terrible") are counted as "nearly correct".} were allowed to participate in further experiments.
Of the 550 crowdsourced workers surveyed, 136 met this criteria.
We refer to this set of raters as the "qualified rater pool" below.

Even with a qualification exam, we found raters often disagree on the appropriate label for a given continuation.
% After labeling several tasks ourselves, we, the authors, also often disagreed.
However, when asked to choose which of two continuations was higher quality quality (if any), raters were better aligned.
With this in mind, we choose to analyze \textit{pairs} of ratings given in the same task.
From five absolute ratings, we construct twenty pairwise preference ratings: two per pair of continuations.
If two continuations receive the same label, they are assigned a preference of 0.
If the first continuation is rated higher than the second, a the pair (first, second) is assigned a score of +1 and the pair (second, first) a score of -1.
All analyses comparing multiple decoding methods use this methodology.

Even with the precautions above, care is needed to ensure repeatable results.
To measure this, we performed an ``A/A'' experiment prior to data collection.
This experiment consists of having the same tasks rated by two different pools of raters.
Identical analyses are performed on both rating results, and the experimental setup is deemed valid if conclusions are consistent.
To achieve this, we constructed 150 tasks\footnote{The large-scale experiment includes 1,930 tasks.} using a subset of the context sequences and decoding methods from our primary experiment.
We artificially split the qualified worker pool in two by sending the same tasks for evaluation at midnight and at noon.\footnote{All tasks within each experiment were rated within 4 hours and 1.5 hours, respectively.}
We submit the same set of tasks to both rater pools.
An analysis of results from both sets of ratings (Figure~\ref{fig:sampling-method-vs-hj-consistency}) reveals a statistically consistent preference of top-$p$ over top-$k$ and (local) temperature sampling, and a severe disapproval of random sampling from the model.
%   [FIGURE]: Show sampling method vs. HJ, both sides of A/A experiment. https://colab.corp.google.com/drive/1EYVXoxwuVbqd8ZsYZ7Oe-PDe6KlYMIO8#scrollTo=Fl4i_cOw0KvO&line=24&uniqifier=1
These results are also consistent with the same statistics gathered in the full-scale experiment presented in the main text and another experiment described below.
%   [FIGURE]: Show sampling method vs. HJ, big experiment. https://colab.corp.google.com/drive/1i3As_tnR-Aam7W2x5vmRQJz1TQdxisl-#scrollTo=msYuUToMp6Tf&line=1&uniqifier=1
%   [FIGURE]: Show sampling method vs. HJ, inter-rater experiment. https://colab.corp.google.com/drive/1EYVXoxwuVbqd8ZsYZ7Oe-PDe6KlYMIO8#scrollTo=1LQzfUrCz7sd&line=3&uniqifier=1

\begin{figure*}[h]
    \centering
    \begin{subfigure}{.20\textwidth}
        \includegraphics[width=\textwidth]{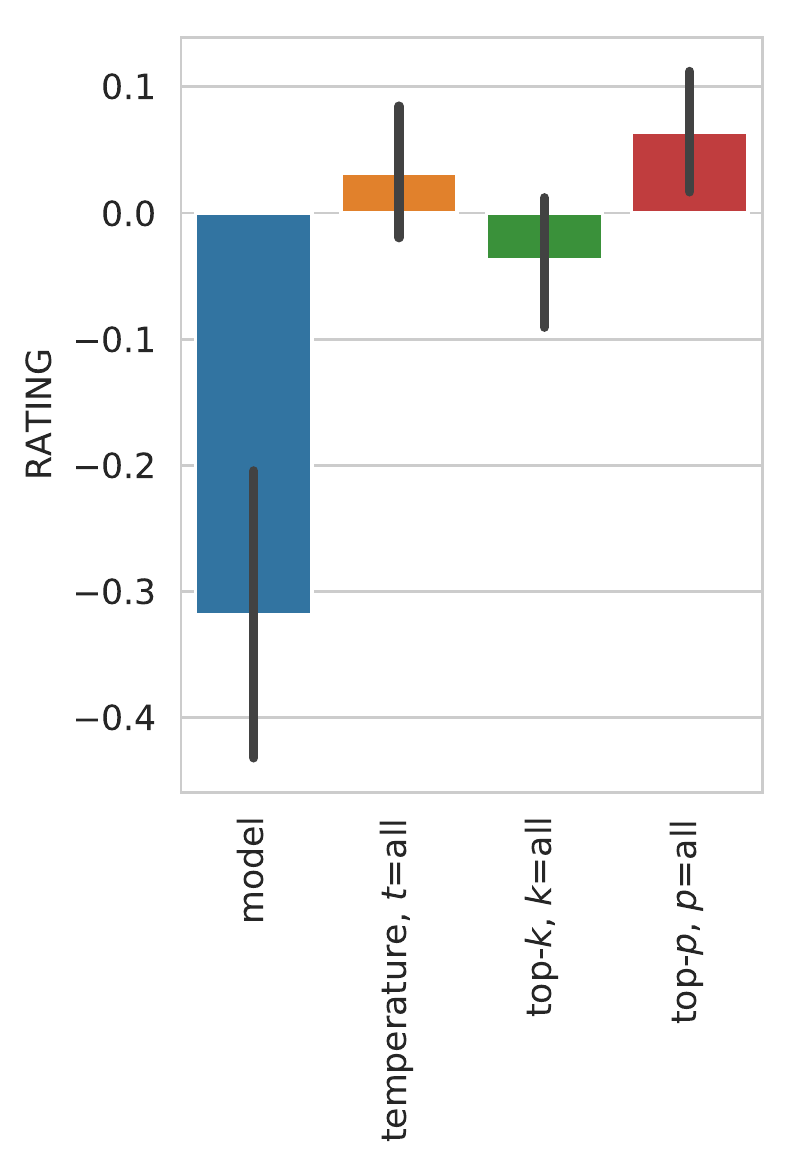}
        \caption{A/A, midnight}
    \end{subfigure}
    \begin{subfigure}{.20\textwidth}
        \includegraphics[width=\textwidth]{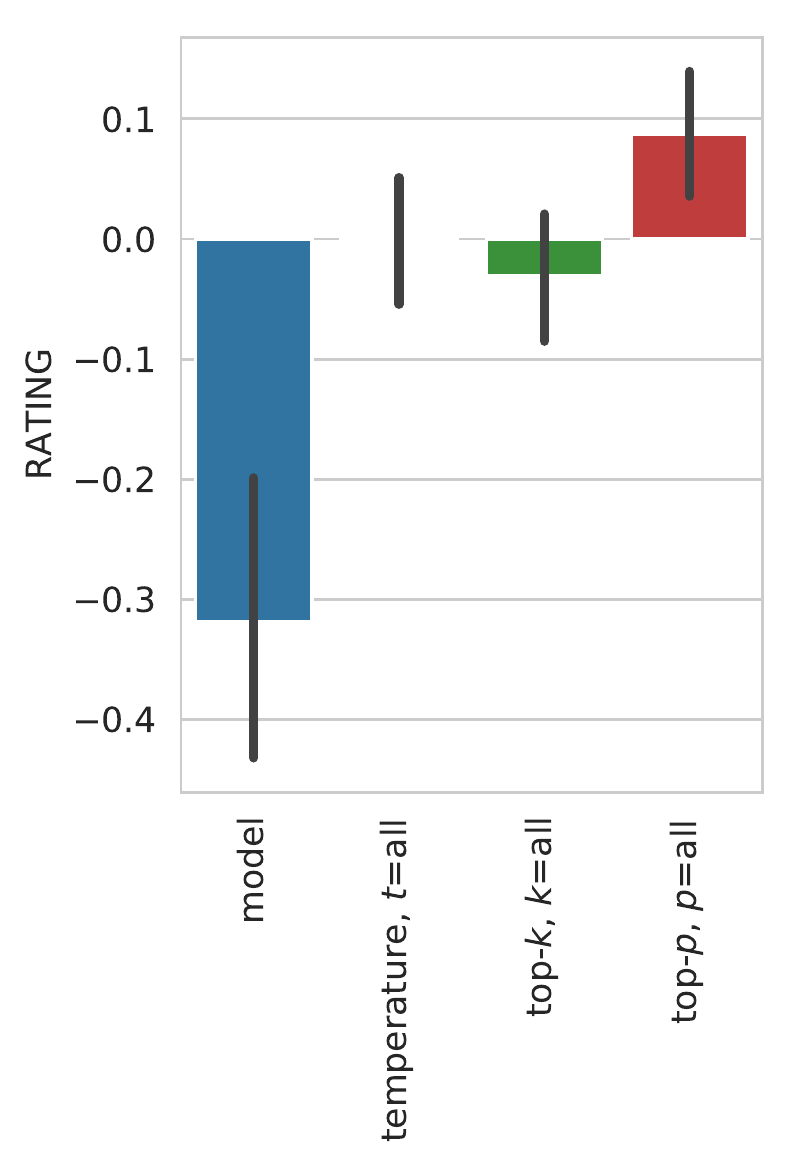}
        \caption{A/A, noon}
    \end{subfigure}
    \begin{subfigure}{.20\textwidth}
        \includegraphics[width=\textwidth]{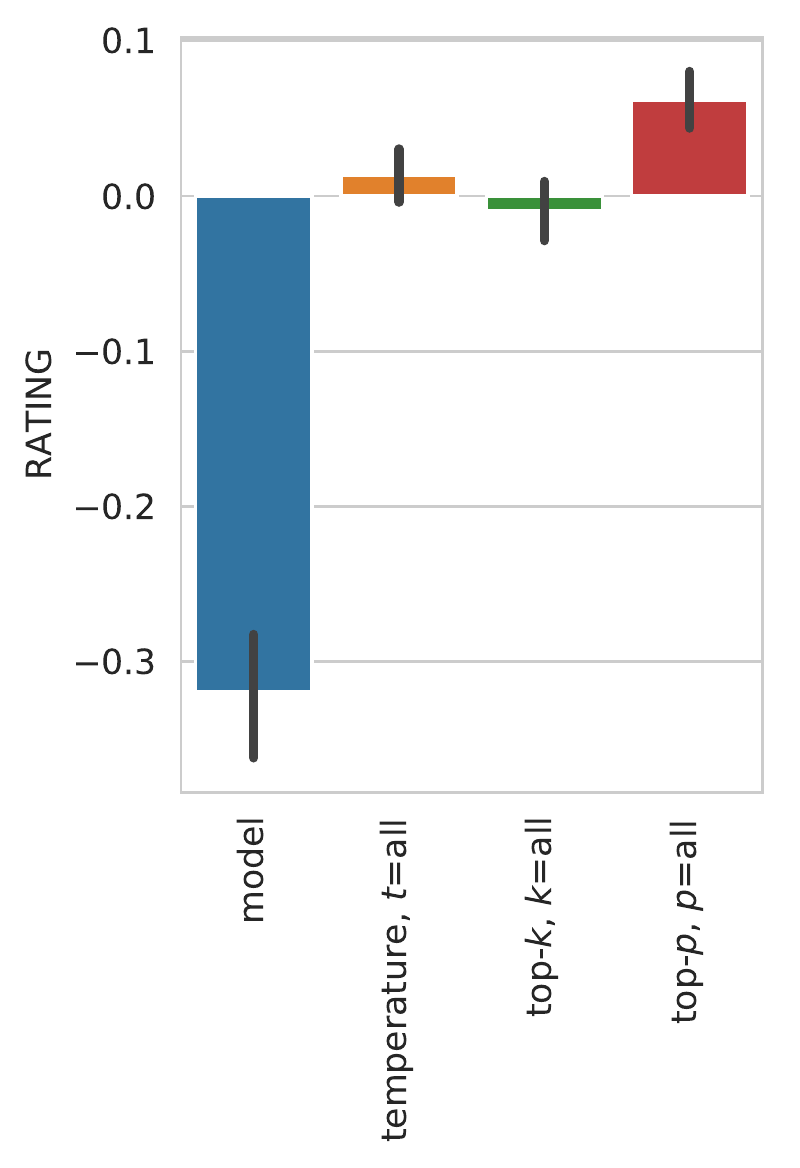}
        \caption{Full-scale}
    \end{subfigure}
    \begin{subfigure}{.20\textwidth}
        \includegraphics[width=\textwidth]{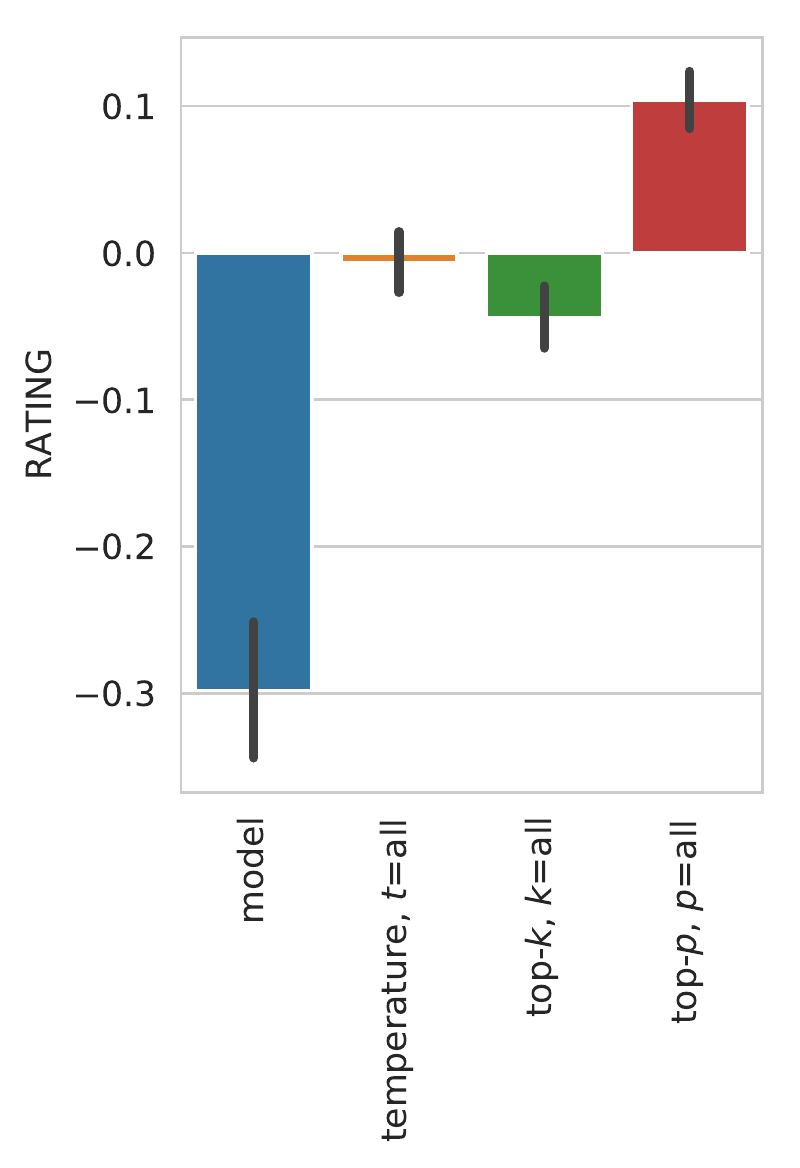}
        \caption{Inter-rater}
    \end{subfigure}
    \caption{Average Human judgement scores for each sampling method, aggregated across sampling method hyperparameters. In spite of being collected by different raters on different sets of tasks and different points in time, rater preference remains consistent.}
    \label{fig:sampling-method-vs-hj-consistency}
\end{figure*}

To further validate the reliability of our methodology, we explicitly measure inter-rater agreement on the same set of 150 tasks in a follow-up experiment after large-scale data collection.
In this experiment, we ask each task be rated by five distinct raters.
We measure Fleiss's Kappa, a measure inter-rater agreement, on the resulting pairwise ratings.
We obtain a score of 0.1964 -- an indication that a correlation between raters exists but that the task is far from unambiguous.
While this may initially appear concerning, we argue that this is an indication of the task's difficulty.
Unlike image classification, for example, a universally agreeable criteria for text quality does not exist.
A measure of Cohen's Kappa on the A/A experiment above produces a score of 0.19578 -- nearly identical to the inter-rater agreement experiment described here.
%   [TABLE]: Present Fleiss's, Cohen's Kappa in a table.
The similarity of these two statistics gives evidence that the proposed experimental design is repeatable in spite of the task's ambiguity.
These reuslts underscore the importance of large-scale, repeatable studies like that presented here.

\begin{figure*}[h]
    \begin{center}
    \begin{tabular}{|c|c|c|}
        \hline
        Experiment & Num Ratings & Kappa \\
        \hline
        A/A & 2,968 & 0.1957 (Cohen's) \\
        Five-Rater & 14,760 & 0.1964 (Fleiss's) \\
        \hline
    \end{tabular}
    \caption{Inter-rater agreement between pairwise preference ratings as measured in a preliminary A/A experiment and an explicit, five-raters-per-task inter-rater agreement experiment. While agreement is low, Kappa is strongly consistent between both experiments.}
    \label{tab:inter-rater-agreement}
    \end{center}
\end{figure*}

We conclude by measuring rater preference between each pair of sampling method and hyperparameter on the five-raters-per-task inter-rater agreement experiment described above.
Results, as shown in Figure~\ref{fig:sampling-config-vs-hj-consistency}, indicate that the same trends presented in the full-scale experiment (Figure~\ref{fig:sampling-method-vs-hj}) hold,
%   [FIGURE]: Show sampling config vs. HJ, inter-rater experiment. Compare to similar for big experiment. https://colab.corp.google.com/drive/1EYVXoxwuVbqd8ZsYZ7Oe-PDe6KlYMIO8#scrollTo=m9f2qB5RzzXo&line=14&uniqifier=1
\begin{itemize}
  \item Top-$p$ is preferred to all other sampling methods,
  \item Increased diversity correlates with lower human judgement scores, and
  \item Random sampling directly from the model produces the lowest human judgement scores by a large margin
\end{itemize}

\begin{figure*}[h]
    \centering
    \includegraphics[width=0.8\textwidth]{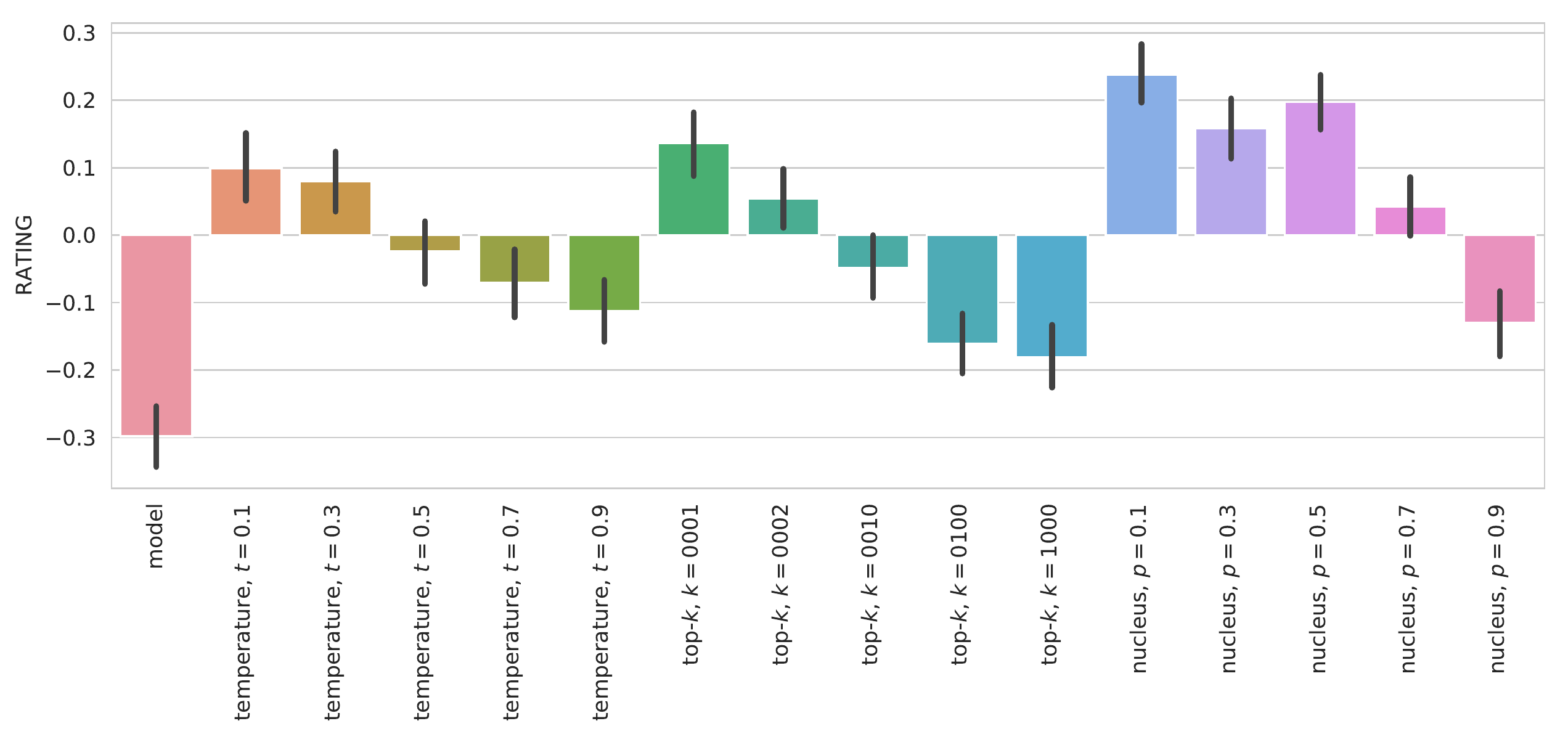}
    \caption{Human judgement scores for each decoding algorithm and hyperparameter choice, as measured in the inter-rater agreement experiment. Preference between sampling methods remains consistent with large-scale experiment shown in Figure~\ref{fig:sampling-method-vs-hj} in spite of using only decodes generated by a subset of context sequences.}
    \label{fig:sampling-config-vs-hj-consistency}
\end{figure*}

% \begin{figure*}[h]
%     \centering
%     \includegraphics[width=\textwidth]{figures/old_turk.png}
%     \caption{Earlier version of the crowdworker task used only to collect data for the likelihood trap figure. }
%     \label{fig:example-task-old}
% \end{figure*}

\end{document}